\documentclass[twoside]{article}

\usepackage[preprint]{aistats2027}
\usepackage{booktabs}       % professional-quality tables
\usepackage{amsfonts}       % blackboard math symbols
\usepackage{nicefrac}       % compact symbols for 1/2, etc.
\usepackage{xcolor}         % colors
\usepackage{multirow}
\usepackage{subcaption}
\usepackage{caption}
\usepackage{graphicx}
\usepackage{amssymb}
\usepackage{amsmath}
\usepackage{enumitem}
\usepackage{graphicx}
\usepackage{amsmath, amssymb}
\usepackage{algorithm}
\usepackage{algorithmic}
\usepackage{float}
\usepackage{natbib}
\usepackage{tabularx}
\usepackage{url}
\usepackage{changes}
\usepackage[expansion=false]{microtype}
\newcommand{\ones}{\mathbf 1}
\newcommand{\reals}{{\mbox{\bf R}}}

\newcommand{\mnorm}[1]{{\vert\kern-0.3ex\vert\kern-0.3ex\vert #1 
\vert\kern-0.3ex\vert\kern-0.3ex\vert}}

\newcommand{\argmin}{\operatornamewithlimits{argmin}}

\usepackage{bbm}
\newcommand{\E}{{\mathbf{E}}}

\newcommand{\prob}{{\mathbbmss{P}}}

\newcommand{\braket}[1]{\langle #1 \rangle}

\newcommand{\eg}{{\it e.g.}}
\newcommand{\ie}{{\it i.e.}}

{\catcode`p =12 \catcode`t =12 \gdef\eeaa#1pt{#1}}      % Get slantfactor
\def\accentadjtext#1{\setbox0\hbox{$#1$}\kern   % Convert it with height
                 \expandafter\eeaa\the\fontdimen1\textfont1 \ht0 }
\def\accentadjscript#1{\setbox0\hbox{$#1$}\kern % Convert it with height
                 \expandafter\eeaa\the\fontdimen1\scriptfont1 \ht0 }
\def\accentadjscriptscript#1{\setbox0\hbox{$#1$}\kern   % Convert it with height
                 \expandafter\eeaa\the\fontdimen1\scriptscriptfont1 \ht0 }
\def\accentadjtextback#1{\setbox0\hbox{$#1$}\kern       % Convert it with height
                 -\expandafter\eeaa\the\fontdimen1\textfont1 \ht0 }
\def\accentadjscriptback#1{\setbox0\hbox{$#1$}\kern     % Convert it with height
                 -\expandafter\eeaa\the\fontdimen1\scriptfont1 \ht0 }
\def\accentadjscriptscriptback#1{\setbox0\hbox{$#1$}\kern % Convert it with height
                 -\expandafter\eeaa\the\fontdimen1\scriptscriptfont1 \ht0 }

\makeatletter
\def\mdots@{\mathinner.\nonscript\!.%
  \ifx\next,.\else\ifx\next;.\else\ifx\next..\else
  \nonscript\!\mathinner.\fi\fi\fi}
\let\ldots\mdots@
\let\cdots\mdots@
\let\dotso\mdots@
\let\dotsb\mdots@
\let\dotsm\mdots@
\let\dotsc\mdots@
\def\vdots{\vbox{\baselineskip2.8\p@ \lineskiplimit\z@
     \kern6\p@\hbox{.}\hbox{.}\hbox{.}\kern3\p@}}
\def\ddots{\mathinner{\mkern1mu\raise8.6\p@\vbox{\kern7\p@\hbox{.}}%
     \raise5.8\p@\hbox{.}\raise3\p@\hbox{.}\mkern1mu}}
\makeatother

\RequirePackage{amsthm}
\newtheoremstyle{descriptive}%
	{\topsep}   %{\medskipamount}          % Space above
	{\topsep}   %  {\medskipamount}          % Space below
	{\rmfamily} % Body font
	{}          % Indent
	{\bfseries} % Head font
	{.}         % Punctuation after thm head
	{ }         % Space after thm head
	{}          % Thm head spec(?)
\newtheoremstyle{propositional}%
	{\topsep}   %  {\medskipamount}          % Space above
	{\topsep}   %  {\medskipamount}          % Space below
	{\itshape}  % Body font
	{}          % Indent
	{\bfseries} % Head font
	{.}         % Punctuation after thm head
	{ }         % Space after thm head
	{}          % Thm head spec(?)
\theoremstyle{propositional}
\newtheorem{thm}{Theorem}[section]

\newtheorem{lem}[thm]{Lemma}

\theoremstyle{descriptive}

\newtheorem{assumption}[thm]{Assumption}

\newtheoremstyle{remarkstyle}%
    {\topsep}    % Space above
    {\topsep}    % Space below
    {\rmfamily} % Body font
    {}          % Indent amount
    {\itshape} % Head font (italic, but not bold)
    {.}         % Punctuation after theorem head
    { }         % Space after theorem head
    {}          % Theorem head spec

\theoremstyle{remarkstyle}
\newtheorem{remark}[thm]{Remark}

\begin{document}



\twocolumn[

\aistatstitle{Data-Driven Laplacian\added{-}Constrained\\ Personalized Federated Learning}

\aistatsauthor{ Samuel Erickson \And Mikael Johansson }

\aistatsaddress{ KTH Royal Institute of Technology \And KTH Royal Institute of Technology } ]

\begin{abstract}
    Personalization addresses one of the defining challenges in federated
    learning: statistical heterogeneity among clients. We show that the
    dissimilarity between the \emph{ideal personalized models} of two clients
    can be upper bounded by the dissimilarity of their distributions. Motivated
    by this, we propose \textsc{Karula}, a data-driven regularization method
    that constrains pairwise model dissimilarities between clients according to
    their distributional differences. Distributional similarity is measured
    using a surrogate for the 1-Wasserstein distance adapted for the federated
    setting, allowing the method to adapt to highly complex interrelations
    between clients that approaches based on hard clustering may fail to
    recover.  We propose an approximate projected stochastic gradient algorithm
    to solve the resulting constrained optimization problem, and prove that it
    converges at a rate of $O(1/K)$ to a neighborhood of a stationary point.
\end{abstract}

\section{Introduction}

Federated learning (FL) is a paradigm for training models on decentralized data
that remains with clients in a network coordinated by a central server.
Developed in response to the growing prevalence of private, decentralized data
sources (such as mobile devices and hospital databases) FL enables collaborative
and communication-efficient model training without requiring the exchange of
client data.  One example of an FL algorithm is Federated Averaging
\citep{McMahan17} where clients perform local gradient steps on a global model
and send updated model parameters to the server. The server then averages these
and returns the updated global model to the clients. FL has been successfully
deployed in a wide range of applications, including mobile keyboard prediction
\citep{Hard18}, autonomous driving \citep{Nguyen22}, and healthcare
\citep{Qayyum22}. However, FL faces several challenges that arise due to client
heterogeneity, be it in terms of models, hardware, or data. In the case of data
heterogeneity, the client data follow different statistical distributions and
may have varying numbers of observations. This means that the performance of
standard FL systems, which learn a single global model, can suffer
significantly in terms of both convergence and generalization. At the same
time, due to limited data, non-collaboratively training separate models also
leads to poor generalization. 

Therefore, in settings where data heterogeneity is significant, federated
learning systems should ideally learn personalized models that borrow strength
from others while respecting local nuances in their data.  Personalized
federated learning (PFL) methods aim to accomplish this through two primary
approaches: (1) globally anchored personalization, where local models leverage
knowledge from a central global model; and (2) peer-based personalization, where
local models learn from similar clients identified at the server level.
Globally anchored methods are appropriate in applications where we expect
each client to have limited dissimilarity with any other client in the network,
such as in next-word prediction on mobile devices. However, in applications such
as recommender systems, distributional differences between two given clients can
be very large. Peer-based methods address this issue by attempting to only learn
from clients that are similar. One approach to peer-based personalization is to
cluster clients, prior to training or concurrently. The cluster model of
statistical heterogeneity does however fail to capture more complex overlaps
between clients. As an example, in a music recommendation application, suppose
one user (client) is only interested in jazz and classical, another in jazz and
funk, and yet another in funk and rock. Clearly, no hard clustering of these
users will capture these relationships accurately.  Other approaches to
peer-based personalization allow for more complex relationships between the
local models of clients, but to the best of our knowledge no prior work
incorporates distributional similarity measures in a principled manner. 

Motivated by this, we develop a PFL method that aims at regularizing the
models by leveraging the following intuition:
\begin{center}
    \it Clients with similar distributions\\should learn similar models.
\end{center}
To do this, the method uses approximations of the Wasserstein distance to
measure distributional similarity between clients, and in turn, leverages
similarity by constraining the dissimilarities between the personalized models.
Our method approximates the Wasserstein distances in a manner that does not
necessitate the exchanging of client data.

% We write $z \sim \mu$ if $\prob(z \in A) = \mu(A)$ for measurable set $A$, and
% $\E_{z \sim \mu}$.

\paragraph{Contributions.}
The main contributions of this work are threefold:
\begin{itemize}
    \item We show that in important special cases of models and losses, the
        dissimilarities between the ideal personalized models can be bounded by
        the $1$-Wasserstein distances between their distributions. This
        provides a principled motivation for our approach and identifies the
        Wasserstein distance as a natural task-agnostic measure  for defining
        pairwise constraints.

    \item Based on this theoretical insight, we propose a peer-based PFL method
        that learns personalized models that are centrally constrained to have
        limited pairwise dissimilarity. These constraints are derived in a
        principled, fully data-driven manner using approximations of the
        Wasserstein distances which do not require the exchange of client data.

    \item We develop an efficient approximate projected stochastic gradient
        method tailored to our framework. It employs variance reduction to deal
        with partial client participation in each communication round. We show
        convergence for smooth, possibly non-convex losses at rate $O(1/K)$ to
        an $\varepsilon$-stationary set of personalized models, where the error
        $\varepsilon$ arises only due to stochastic gradients from the clients
        and approximate projection at the server level.

\end{itemize}

\paragraph{Notation.} We define the Euclidean norm $\|\cdot\|_2$ for parameter
vectors $\theta = (W_1, \dots, W_b)$ via $\|\theta\|_2^2 = \sum_{l=1}^L
\|W_l\|_F^2$. For a vector $h \in \reals^d$, we define the RMS-norm
$\|h\|_\text{RMS} = \|h\|_2 / \sqrt{d}$. For weight matrices $W \in \reals^{d'
\times d}$, we define the RMS-to-RMS norm by $\|W\|_{\text{RMS} \to \text{RMS}}
= \sup_{\|h\|_\text{RMS} \leq 1} \|Wh\|_\text{RMS} = \sqrt{d/d'} \times
\sup_{\|h\|_2 \leq 1} \|Wh\|_2$. 

\section{Related work}
\label{sec:related-work}

\paragraph{Globally anchored personalization.} A common approach to PFL is to
learn local models that are, in some sense, close to a shared global model.
\cite{Fallah20} take this approach using meta-learning to formulate a method
that learns a global model, from which each client takes a few local gradient
descent steps to produce the local model. Additionally, they show a connection
between the convergence of their method and statistical heterogeneity measured
in total variation and Wasserstein distance, respectively. Taking a
regularization approach to the problem, \citet{Dinh20} and \citet{Li21} propose
methods that penalize the squared norm difference between the local and the
global model. Another approach is given by \cite{Deng20}, whose algorithm
learns personalized models that are convex combinations of a shared global
model and a completely local model. Other works explore architectural
approaches to globally anchored personalization. \cite{Collins21} let clients
share feature extractors while a fully connected layer is personalized by each
client. \cite{Shamsian21} and \cite{Chen22} take a more complex approach, where
a central hypernetwork is trained to produce a model for every client which
balances effective parameter sharing and personalization.

\paragraph{Peer-based personalization.} A more flexible personalization method
is to identify clients that are similar and have them borrow strength from each
other directly rather than from a single aggregated model. One way to do this
is by clustering the clients. The PFL method proposed by \cite{Ghosh20} does
this iteratively by alternating between optimizing the cluster models and
clustering the clients, which is done by finding the cluster whose model
achieves the smallest local loss. \cite{Armacki22} also propose a clustering
based approach, opting for convex clustering via a group-lasso penalty on the
differences of models. To allow for more flexible personalization, other works
have proposed approaches based on mixtures of distributions and soft clustering
\citep{Marfoq21, Ruan22, Dinh22, Wu23, Lin24, Hashemi24}. \citet{Ye23} infer a
collaboration graph during training in order to promote cosine similarity
between client models which have a strong connection according to the graph.

\paragraph{Similarity measures.} Measures of task and distribution similarity
are of great interest in several learning contexts, particularly transfer
learning \citep{Zhong18, Pan20}, multi-task learning \citep{Shui19}, and
meta-learning \citep{Zhou21}. These measures quantify how closely related
different tasks or data distributions are, influencing decisions on model
adaptation, aggregation, and generalization. Common approaches include
statistical dissimilarity measures such as the Kullback-Leibler divergence
\citep{Kullback51}, Jensen-Shannon divergence \citep{Lin91}, and Wasserstein
distance \citep{Kantorovich60}. In addition, similarity can be assessed through
feature-space representations using methods like cosine similarity or
kernel-based measures. In the context of FL, \cite{Zhao18} used the Wasserstein
distances between the local client distributions and the global distribution to
quantify the effect of statistical heterogeneity on local iterations.
\cite{Fallah20} similarly use the distance to derive an upper bound on the
parameter in the common bounded gradient dissimilarity assumption.

\section{Proposed method}
\label{sec:method}

In this section, we introduce \textsc{Karula},\footnote{The name is a
portmanteau of Kantorovich, Rubinstein and Laplace.} a method for federated
learning of personalized models where model differences are constrained by the
distributional similarity of client data. This reinforces the intuitive notion
that clients with similar data should learn similar models. To support partial
client participation during communication rounds, we introduce an efficient
projected stochastic gradient method with variance reduction. 
% This method is also suitable for use in other FL methods.

\subsection{Main idea}\label{ssec:main-idea}

Consider an FL setting with $n$ clients, each able to communicate with a
central server. Suppose $Z = {X} \times {Y}$ is the data space, with ${X}$ and
${Y}$ being the feature and label space, respectively. To each client $i \in
\{1,\dots,n\}$, we associate an underlying distribution $\mu_i \colon
\mathcal{F} \to \reals_+$ for some $\mathcal{F} \subset 2^{Z}$. Each client
holds a dataset $\mathcal{D}_i$ of $N_i$ independent samples drawn from
$\mu_i$. Moreover, suppose we have a prediction space $A$. A common method in
PFL is to have the clients share the feature maps and personalize only the
output layer \citep{Arivazhagan19,Ghosh20}, so for this reason, we consider
models $\phi_{\theta_i}: X \to A$ parameterized by $\theta_i = (W_1, \dots,
W_{b-1}, W_b^{(i)})$ of the form
\[
    \begin{aligned}
        \phi_{\theta_i}(x) &= W_b^{(i)} h_{b-1}(x), \\
        h_l(x)             &= \sigma(W_{l} h_{l-1}), \quad l = 1, \dots, b-1, \\
    \end{aligned}
\]
with $h_0(x) = x$ and activation function $\sigma$. (We omit biases for clarity
of exposition.) We call $h_{b-1}$ the feature map. With $b=1$ we obtain a
linear model, whereas with $b>1$ we obtain a general MLP. The learning goal is
to find personalized model parameters $\theta_i$ for each client $i$ that
minimizes their local objective $f(\theta_i; \mu_i) = \E_{(x,y) \sim \mu_i}
[\ell(\phi_{\theta_i}(x), y)]$, where $\ell \colon A \times Y \to \reals$ is a
loss function. That is, we would like to solve 
\begin{equation}\label{eq:pfl-prob}
    \begin{array}{ll} 
        \mbox{minimize} & \frac{1}{n} \sum_{i=1}^n f(\theta_i; \mu_i).
    \end{array}
\end{equation}
However, we only have access to small local datasets $\mathcal{D}_i$ and not
the distributions $\mu_i$, and simply minimizing the loss on these limited
datasets typically leads to poor generalization. To address this challenge, we
introduce a regularization method that couples the learning tasks across
clients by approximately solving a constrained optimization problem:
\begin{equation}\label{eq:karula-prob}
    \begin{array}{ll} 
        \mbox{minimize} & F(\theta) = \frac{1}{n} \sum_{i=1}^{n} f(\theta_i; \mathcal{D}_i) \\
        \mbox{subject to} & \|\theta_i - \theta_j\|_2^2 \leq t D_{ij}, \quad i,j \in \{1,\dots,n\}.
    \end{array}
\end{equation}
Here $t \geq 0$ is a regularization parameter, and each $D_{ij} \geq 0$
quantifies the dissimilarity between clients $i$ and $j$, leveraging the
intuition that \emph{similar clients should learn similar models}. This can be
interpreted as a hard-constrained analogue of graph Laplacian regularization,
enforcing bounded variation across the graph. The question is, how should we
choose the dissimilarity parameters $D_{ij}$? 

% \paragraph{Connections to other methods.} The formulation
% \eqref{eq:karula-prob} is in spirit closely related to Laplacian regularized
% models \citep{Ando06}
% \[
%     \begin{array}{ll} 
%         \underset{\theta_i \in \mathbf{R}^d}{\mbox{minimize}} & \displaystyle \sum_{i=1}^n f_i(\theta_i) + \sum_{i,j} w_{ij} \|\theta_i - \theta_j\|_2^2.
%     \end{array}
% \]
% In fact, using convex duality, it is clear that they are equivalent if the
% weights $w_{ij}$ in the Laplacian regularized problem coincide with the optimal
% Lagrange multipliers associated with the constraints in \eqref{eq:karula-prob}. 

\paragraph{Model and distributional dissimilarity.} To quantify the
relationship between data distributions and the corresponding learned models,
we consider the \emph{1-Wasserstein distance}---also known as the \emph{earth
mover's distance} or the \emph{Kantorovich–Rubinstein distance}.  Suppose we
equip the data space $Z$ with the metric $d_Z$. Then the 1-Wasserstein distance
between two distributions $\mu_i$ and $\mu_j$ is 
\[
    \mathcal{W}_1(\mu_i, \mu_j) = \inf_{\pi \in \Pi(\mu_i, \mu_j)} \E_{(z, z')\sim \pi}[d_Z(z, z')], %d\gamma(\xi, \xi')
\]
where $\Pi(\mu_i, \mu_j)$ is the set of all probability measures with marginals
$\mu_i$ and $\mu_j$. The problem of minimizing this expected distance is an
instance of an \emph{optimal transport problem}, and the distance it defines is
a proper metric on the space of distributions with finite first moment on $Z$
\citep{Villani08}. 

We will now show that the pairwise dissimilarities between the ideal model
parameters $(\theta_1^\star, \dots, \theta_n^\star)$ that solve
\eqref{eq:pfl-prob} can be bounded by the 1-Wasserstein distance between the
respective distributions. 

First, we assume that the loss under the \emph{ideal local models} are
Lipschitz continuous \emph{with respect to the data}.
\begin{assumption}[Lipschitz continuity]\label{asmp:lipschitz}
    The loss $\ell(\phi_{\theta_i^\star}(x), y)$ under the ideal models is
    $L_\ell$-Lipschitz in $(x,y) \in Z$ with respect to the metric $d_Z$.
\end{assumption}
This is satisfied in several important special cases. For instance, suppose the
data is uniformly bounded, so there exists some $R_X \geq 0$ such that
$\|x\|_\text{RMS} \leq R_X$ for all $x \in X$. Considering an image
classification task with normalized images, we have $R_X = 1$. Second, suppose
the activation function $\sigma$ is $1$-Lipschitz and satisfies $\sigma(0)=0$,
as for ReLU activations. Then a uniform upper bound on the Lipschitz constants
of the ideal models is
\[
    L_\phi := \max_i \prod_{l=1}^b
    \|W_l^{(i)\star}\|_{\text{RMS} \to \text{RMS}}.
\]
Under the spectral scaling condition proposed by \cite{Yang23}, this bound is
$O(1)$ as widths grow at fixed depth, provided the ideal weights satisfy that
condition. Now, consider the metric
\[
    d_Z((x,y), (x',y')) = \|x - x'\|_\text{RMS} + \mathrm{TV}(y,y'),
\] 
where $\mathrm{TV}(\cdot,\cdot)$ is the total variation metric. Then, with
softmax cross-entropy and $K$ classes, Assumption~\ref{asmp:lipschitz} holds
with $L_\ell \leq \sqrt{2K} L_\phi\max\{1,R_X\}$. Note again that this only
needs to hold for the \emph{ideal models}, so $L_\phi$ is finite whenever their
weight matrices have finite norms.

Second, we assume a quadratic growth bound between the chosen \emph{ideal
output layers} under the \emph{ideal shared feature map} $h_\star$.
\begin{assumption}[Quadratic growth]\label{asmp:quadratic-growth}
    There exists $\gamma > 0$ such that, for all $i,j \in \{1,\dots,n\}$,
    \[
        \begin{aligned}
            f(\theta_j^\star;\mu_i) - f(\theta_i^\star;\mu_i) \geq \gamma \|W_b^{(i)\star} - W_b^{(j)\star}\|_F^2.
        \end{aligned}
    \]
\end{assumption}
For softmax cross-entropy with $K$ classes, we can assume centered
representatives satisfying
\[
    \mathbf{1}^\top W_b^{(i)\star} = 0 \quad \text{for all } i.
\]
This removes the common-logit shifts that leave the loss unchanged. Define the
second-moment matrix $\Sigma_i = \E_{(x,\cdot)\sim\mu_i} [h_\star(x)
h_\star(x)^\top]$ of the ideal feature map under the $i$th client distribution.
Suppose that the feature map does not collapse the data, meaning there exists a
$\lambda > 0$ such that $\lambda \leq \min_i \lambda_{\min}(\Sigma_i)$. Then
Assumption~\ref{asmp:quadratic-growth} holds with
\[
    \gamma = \frac{\lambda e^{-\sqrt{2K}L_\phi R_X}}{2K}.
\]

\begin{thm}\label{thm:dissim-bound}
    Suppose Assumptions~\ref{asmp:lipschitz} and \ref{asmp:quadratic-growth}
    hold. Then the ideal output layers satisfy
    \[
        \|W_b^{(i)\star} - W_b^{(j)\star}\|_F^2 \leq \frac{L_\ell}{\gamma} \mathcal{W}_1(\mu_i, \mu_j).
    \]
\end{thm}
\begin{proof}
    The proof of this and forthcoming results is presented in
    Appendix~\ref{sec:proof}. 
\end{proof}

% While this theorem does not give guarantees for empirical Wasserstein
% distances, it provides a principled motivation for using the Wasserstein
% distance over any other metric on the space of probability distributions.

\subsection{Calculating the dissimilarity parameters}
\label{ssec:data-embedding}

With the present motivation in mind, it is natural to set the dissimilarity
parameters proportional to the Wasserstein distances $\mathcal{W}_1(\hat \mu_i,
\hat \mu_j)$ between the empirical distributions $\hat \mu_i$. However,
directly doing this would involve exchanging client data and solving $O(n^2)$
optimal transport problems---an operation that is computation and
communication-intensive. For these reasons, we propose to approximate the
Wasserstein distances via linear embedding of the datasets, following prior
work such as~\cite{Kolouri20} and \cite{Liu25}. Specifically, we introduce a
reference dataset $\mathcal{D}_0 = (z_1^0, \dots, z_{N_0}^0) \in \reals^{N_0
\times d}$ of $N_0$ i.i.d.\ samples drawn from a reference distribution $\mu_0$
(we may take \eg, a standard normal or uniform distribution), where $d$ is the
dimension of the data space. The optimal transport plan $\pi_i^\star$ between
the empirical distributions $\hat \mu_0$ and $\hat \mu_i$ is then
\[
    \pi_i^\star = \argmin_{\pi \in \Pi(\hat \mu_0, \hat \mu_i)} \sum_{j=1}^{N_0} \sum_{k=1}^{N_i} \pi_{jk} d_Z (z_j^0, z_k^i).
\]
By forming $M_i = N_0 \cdot \pi_i^\star \mathcal{D}_i$, we can define an
embedding $E \colon \mathcal{D}_i \mapsto (M_i - \mathcal{D}_0) / N_0$ that has
the property that $\|E(\mathcal{D}_i) - E(\mathcal{D}_j)\|_1$ approximates the
Wasserstein distance $\mathcal{W}_1(\hat \mu_i, \hat \mu_j)$, and in fact
defines a pseudometric on the space of datasets whenever $N_0 \geq \max_i
N_i$.\footnote{$M_i$ is an approximation of the \emph{Monge map}
\citep{Villani08}.} We thus choose the dissimilarity parameters as $D_{ij} =
\|E(\mathcal{D}_i) - E(\mathcal{D}_j)\|_1$. By doing this, each client only has
to compute a single optimal transport plan. On the communication side, only a
reference data set and the embeddings need to be transmitted over the network
once, \ie, an $O(N_0d)$ communication complexity for each client once prior to
training. In cross-device FL, this will typically be on the order of a single
model update or gradient since there is no need to choose $N_0$ greater than
$\max_i N_i$.  

\begin{remark}
    Our focus in this paper is on statistical efficiency. However, compared to
    other FL algorithms, the distance computation exposes a summary of each
    client’s data to the server which can expose private information.
    Appendix~D shows that adding Gaussian noise to the embeddings makes their
    release (and every distance computed from it) differentially private. The
    guarantee covers the distance computation only; end-to-end privacy would
    require to also privatize the training updates as in, for example,
    DP-FedAvg \citep{McMahan18}.
\end{remark}

\subsection{Algorithm}
\label{ssec:algorithm}
To solve the \textsc{Karula} problem \eqref{eq:karula-prob}, we introduce a
projected gradient algorithm outlined in Algorithm~\ref{algo:karula} that
supports partial client participation in each training round. To eliminate the
error introduced by this partial participation, our algorithm uses variance
reduction, taking inspiration from methods such as \textsc{Saga}
\citep{Defazio14} and \textsc{Fedvarp} \citep{Jhunjhunwala22}. Our approach can
be viewed as a relaxation of the standard distributed stochastic gradient
method: rather than enforcing full model agreement, each iteration shrinks model
dissimilarities based on pairwise constraints by projecting onto the constraint
set 
\[
    \mathcal{K} = \{ \theta \colon \|\theta_i - \theta_j\|_2^2 \leq t D_{ij} \}.
\]
In particular, in each training round, a subset of clients returns stochastic
gradients $G_i$ to the server, which then performs a variance reduced gradient
step with respect to $F$, and approximately projects the iterate onto the
constraint set $\mathcal{K}$. In practice, we may also take $G_i$ to return
pseudogradients, \eg, computed via several local iterations using ADAM or any
other optimizer. Moreover, if new clients join the network, the server can
easily compute and transmit an averaged model.  

\begin{algorithm}
\caption{\textsc{Karula}}
\begin{algorithmic}
    \label{algo:karula}
    \STATE \textbf{Setup:} Initializations $(\theta_1^0, \dots, \theta_n^0)$,
    step size $\eta$, participation size $s$, number of iterations $K$. 

    \STATE Server receives $\theta_i^0$ and $G_i(\theta_i^0)$ from clients
    $i\in\{1,\dots,n\}$, and forms $\theta^0 = (\theta_1^0, \dots,
    \theta_n^0)$, $\alpha^0 = \theta^0$, and $g^0 = (G_1(\theta_1^0), \dots,
    G_n(\theta_n^0))$.

    \FOR{$k = 0, \dots, K-1$}
        \STATE Server samples uniform subset $S_k \subset \{1,\dots,n\}$ of $s$
        distinct clients and sends $\theta_i^{k}$ to client $i \in S_k$.

        \STATE Client sets
        \[
            \alpha_i^{k+1} = \begin{cases}
                \theta_i^{k} & \text{if } i \in S_k, \\
                \alpha_i^k     & \text{otherwise}. 
            \end{cases}
        \]

        \STATE Clients $i\in S_k$ return fresh stochastic gradients
        $G_i(\alpha_i^{k+1})$ to the server.

        \STATE Server sets
        \[
            g_i^{k+1} = \begin{cases}
                G_i(\alpha_i^{k+1}) & \text{if } i \in S_k, \\
                g_i^k               & \text{otherwise}. 
            \end{cases}
        \]
        and forms $\nu^k = g^k/n + (g^{k+1} - g^k)/s$.

        \STATE Server performs approximate projection 
        \[
            \theta^{k+1} = \hat{\Pi}_\mathcal{K}(\theta^k - \eta \nu^k).
        \]

    \ENDFOR
\end{algorithmic}
\end{algorithm}

The projection $\Pi_\mathcal{K} (\vartheta)$ of a collection of models
$\vartheta$ onto $\mathcal{K}$ does not admit an analytical
solution, so we use an iterative method to approximate it. To this end, we
define the \emph{approximate projection} $\hat{\Pi}_\mathcal{K}$ that returns a
point with tolerance $\delta$ in the projection problem. We solve the
projection problem by simply averaging the shared parameters, and running a GPU
optimized instance of ADMM \citep{Boyd11} on the output layer at the server
level. We describe the implementation in detail in
Appendix~\ref{sec:projection}.

\paragraph{Computation and communication complexity.} Supposing there are $p_L$
parameters in the output layer, the ADMM implementation performs $O(p_Ln^2)$
total work per iteration, and under unbounded parallelism the span of an
iteration is $O(n+p_L)$, with a memory complexity of $O(p_Ln^2)$. Thus, if the
model has $p$ parameters in total, \textsc{Karula} does $O(p_L n^2 + pn)$ total
work at the server level per global round with memory complexity $O(p_L n^2 +
np)$. In comparison, \textsc{FedAvg} performs $O(ps)$ total work with memory
complexity $O(ps)$. As for communication, \textsc{Karula} has the same
complexity as \textsc{FedAvg} after the initial computation of the
dissimilarity parameters.

\section{Convergence analysis}
\label{sec:analysis}

In this section, we analyze the convergence properties of our proposed
algorithm for solving Problem~\eqref{eq:karula-prob} in the general smooth case
when the losses have Lipschitz continuous gradients. The analysis is inspired
by the work of \cite{Reddi16}, which we have adapted for the FL setting, and
extended to allow for approximate projections and stochastic gradients of the
individual client losses. 

Our goal in the smooth, possibly nonconvex case, is to find a stationary point,
but since the problem is constrained, the ordinary stationarity condition $0 =
\nabla F(\theta)$ is not appropriate. In lieu of this, the analogous condition
for constrained optimization is that $\theta$ is a zero of the \emph{gradient
mapping} \citep{Nesterov18}
\[
    \mathbf{G}_\eta (\theta) = \frac{1}{\eta} \left( \theta - \Pi_\mathcal{K} (\theta - \eta \nabla F(\theta)) \right).
\]
Informally, we can think of the condition $0 = \mathbf{G}_\eta (\theta)$ as
saying one of two things: either $\theta$ is a feasible point with zero
gradient, or it is a point such that the negative gradient is normal to the
feasible set.

To establish the convergence result, we make three assumptions. First, we
assume the loss is smooth.
\begin{assumption}
    \label{asmp:smooth}
    The local objectives are $L$-smooth, \ie, they are continuously
    differentiable with $L$-Lipschitz continuous gradients: 
    \[
        \|\nabla f(\theta_i; \mathcal{D}_i) - \nabla f(\theta_i'; \mathcal{D}_i)\|_2 \leq L \|\theta_i - \theta_i'\|_2
    \]
    for all $\theta_i, \theta_i'$.
\end{assumption}
Under the product Euclidean norm, the mean objective $F$ is consequently
$L/n$-smooth. Second, we assume the stochastic gradients have bounded
conditional mean squared error.
\begin{assumption}
    \label{asmp:bounded-variance}
    Stochastic gradients have bounded mean squared error, meaning
    \[
        \E\!\left[\|G_i(\theta_i)-\nabla f(\theta_i;\mathcal{D}_i)\|_2^2\right] \leq \sigma^2
    \]
    Note that the oracle \emph{need not} be unbiased.
\end{assumption}
Third, we assume the approximate projection has bounded sub-optimality. This
assumption is standard for inexact proximal operators in general
\citep{Schmidt11}.
\begin{assumption}
    \label{asmp:proj}
    The set $\mathcal{K}$ is nonempty, closed, and convex. For every query
    point $z$, the approximate projection returns a feasible point
    $\hat{\Pi}_\mathcal{K}(z) \in \mathcal{K}$ satisfying
    \[
        \frac{1}{2\eta}\|z-\hat{\Pi}_\mathcal{K}(z)\|_2^2
        \leq \delta + \frac{1}{2\eta}\|z-\Pi_\mathcal{K}(z)\|_2^2
    \]
    for the step size $\eta > 0$ and tolerance $\delta \geq 0$.
\end{assumption}
Under these assumptions, the proposed algorithm enjoys sublinear convergence to
an $\varepsilon$-stationary point where $\varepsilon$ reflects the effect of
stochastic gradients and approximate projections. Here, we assume feasible
initialization only for the sake of exposition, which is trivially satisfied by
assigning the same initial model to all clients.

\begin{thm}\label{thm:nonconvex}
    Under Assumptions~\ref{asmp:smooth}, \ref{asmp:bounded-variance} and
    \ref{asmp:proj}, initialization $\theta^0 \in \mathcal{K}$, participation
    size $1 \leq s \leq n$, and step size $0 < \eta \leq \frac{s^2}{4nL}$, the
    iterates of Algorithm~\ref{algo:karula} satisfy
    \[
        \frac{1}{K} \sum_{k=0}^{K-1} \E\|\mathbf{G}_\eta (\theta^k)\|^2 \leq \frac{2}{1-2\eta L/n} \left( \frac{\Delta}{\eta K} + \varepsilon \right),
    \]
    where $\Delta = F(\theta^0) - \inf_{\theta\in\mathcal{K}} F(\theta) < \infty$
    and $\varepsilon = \frac{4(2n-s)}{ns}\sigma^2 + 2\delta/\eta$.
\end{thm}

In contrast to many other FL algorithms, the convergence of the proposed method
does not depend on any bounds on client heterogeneity or assumptions about the
underlying data distributions. Moreover, while it is common to assume exact
projections in theory, it is not particularly realistic. Our analysis
explicitly accounts for projection errors and quantifies how these affect the
size of the neighborhood that the algorithm converges to in expectation. This
has practical implications, since it informs the choice of the error $\delta$
to trade off the iteration cost with proximity to stationarity.

\section{Numerical experiments}
\label{sec:experiments}

\begin{table*}[t]
\centering
\caption{
    Average 10th, 50th, and 90th percentile client test accuracies on FEMNIST
    and CIFAR datasets over three random seeds. The results including standard
    deviations are given in Table~\ref{tab:complete-results} in
    Appendix~\ref{sec:experimental-details}.
}
\label{tab:main-results}
\resizebox{\textwidth}{!}{%
\begin{tabular}{l|ccc|ccc|ccc|ccc|ccc}
\toprule
& \multicolumn{3}{c|}{\textbf{FEMNIST}}
& \multicolumn{3}{c|}{\textbf{CIFAR10 $\alpha=0.1$}}
& \multicolumn{3}{c|}{\textbf{CIFAR10 $\alpha=0.5$}}
& \multicolumn{3}{c|}{\textbf{CIFAR100 $\alpha=0.1$}}
& \multicolumn{3}{c}{\textbf{CIFAR100 $\alpha=0.5$}} \\
\cmidrule(lr){2-4}
\cmidrule(lr){5-7}
\cmidrule(lr){8-10}
\cmidrule(lr){11-13}
\cmidrule(lr){14-16}

\textbf{Method}
& \textbf{10th} & \textbf{50th} & \textbf{90th}
& \textbf{10th} & \textbf{50th} & \textbf{90th}
& \textbf{10th} & \textbf{50th} & \textbf{90th}
& \textbf{10th} & \textbf{50th} & \textbf{90th}
& \textbf{10th} & \textbf{50th} & \textbf{90th} \\
\midrule

\textsc{FedAvg}
& 71.61 & 83.07 & 91.83
& 44.44 & 63.78 & 78.72
& 59.61 & 70.45 & 79.02
& 31.48 & 38.90 & 46.83
& 31.65 & 37.40 & 42.37 \\

\textsc{pFedMe}
& 46.69 & 60.80 & 74.92
& 59.65 & 81.78 & 98.11
& 43.10 & 58.47 & 73.32
& 30.81 & 43.38 & 51.50
& 25.92 & 32.41 & 39.28 \\

\textsc{IFCA}
& 71.08 & 82.18 & 92.12
& 46.63 & 65.67 & 85.07
& \textbf{60.08} & 69.17 & 76.61
& 31.34 & 39.35 & 48.49
& 32.43 & 36.17 & 41.67 \\

\textsc{FedSoft}
& 40.61 & 63.47 & 76.33
& 59.42 & 82.85 & 96.84
& 43.09 & 60.74 & 73.47
& 39.80 & 49.49 & 57.84
& 12.92 & 36.96 & \textbf{45.27} \\

\textsc{pFedGraph}
& 64.65 & 80.66 & 88.78
& 21.01 & 49.75 & 71.27
& 41.85 & 57.59 & 69.43
& 23.71 & 31.80 & 41.20
& 29.31 & 36.19 & 42.61 \\

\textsc{FedRep}
& 53.49 & 65.74 & 76.24
& \textbf{69.72} & 86.37 & 97.69
& 56.04 & 66.32 & 76.91
& 42.74 & 50.66 & 57.89
& 25.49 & 30.68 & 35.44 \\

\textsc{pFedHN}
& 66.12 & 79.65 & 89.23
& 39.76 & 67.81 & 93.67
& 6.51 & 26.87 & 48.04
& 30.76 & 43.97 & 54.22
& 17.53 & 31.65 & 39.13 \\

\textsc{FedAMP} 
& 50.97 & 66.41 & 79.87
& 63.50 & 82.51 & 96.72 
& 43.42 & 61.28 & 73.81
& 40.49 & 48.23 & 57.96 
& 22.59 & 36.62 & 44.67 \\

\textsc{Ditto} 
& 58.58 & 73.25 & 84.39 
& 64.68 & 85.28 & 97.90 
& 47.46 & 70.23 & 80.62 
& 33.87 & 49.50 & 60.24 
& 24.07 & 38.12 & 44.82 \\

\midrule
\textsc{Karula}
& \textbf{73.53} & \textbf{84.02} & \textbf{92.52}
& 63.23 & \textbf{86.68} & \textbf{99.05}
& 58.84 & \textbf{72.60} & \textbf{82.56}
& \textbf{45.22} & \textbf{53.64} & \textbf{61.39}
& \textbf{33.82} & \textbf{38.44} & 43.06 \\

\bottomrule
\end{tabular}%
}
\end{table*}

\begin{figure*}[ht]
    \centering
    \begin{subfigure}[b]{0.19\textwidth}
        \centering
        \includegraphics[width=\linewidth]{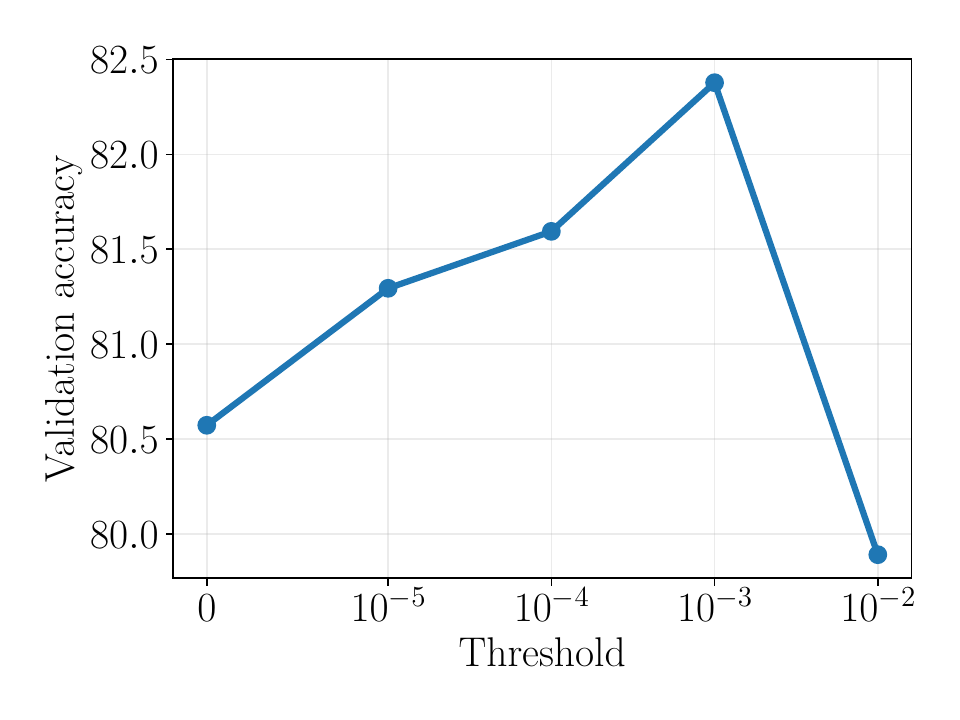}
        \caption{FEMNIST}
        % \label{fig:femnis}
    \end{subfigure}
    \begin{subfigure}[b]{0.19\textwidth}
        \centering
        \includegraphics[width=\linewidth]{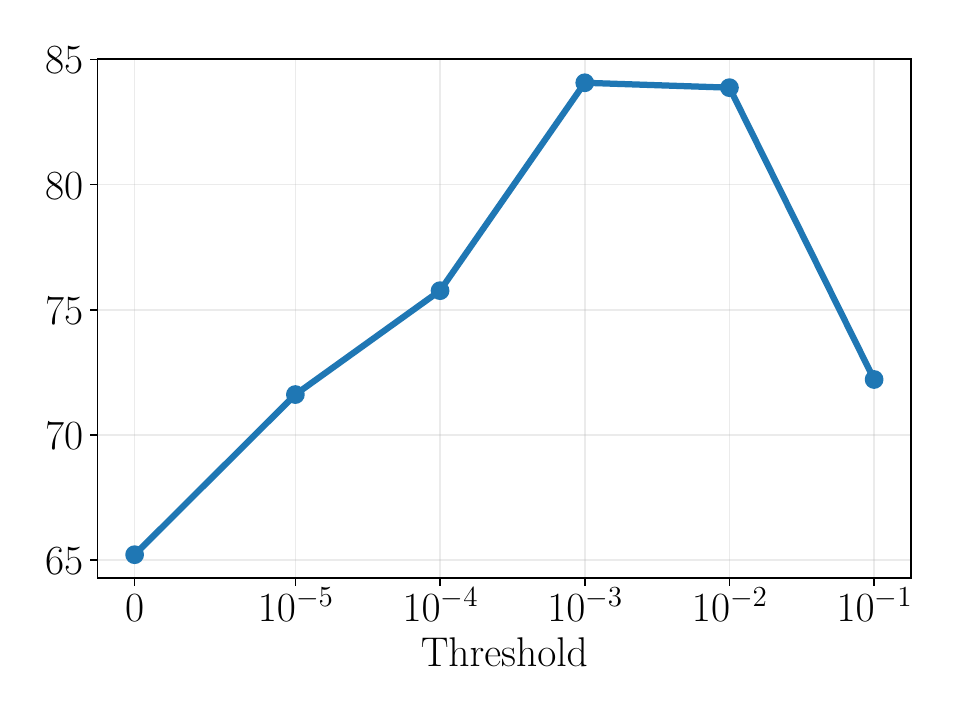}
        \caption{CIFAR10 $\alpha=0.1$}
        % \label{fig:test_accuracy}
    \end{subfigure}
    \begin{subfigure}[b]{0.19\textwidth}
        \centering
        \includegraphics[width=\linewidth]{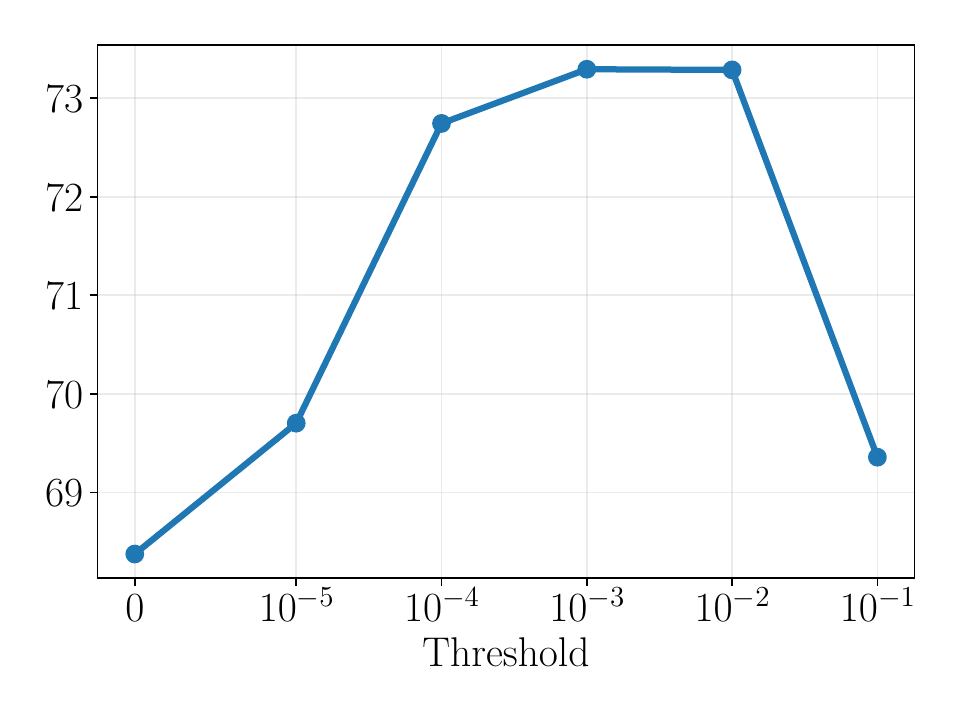}
        \caption{CIFAR10 $\alpha=0.5$}
        % \label{fig:train_loss}
    \end{subfigure}
    \begin{subfigure}[b]{0.19\textwidth}
        \centering
        \includegraphics[width=\linewidth]{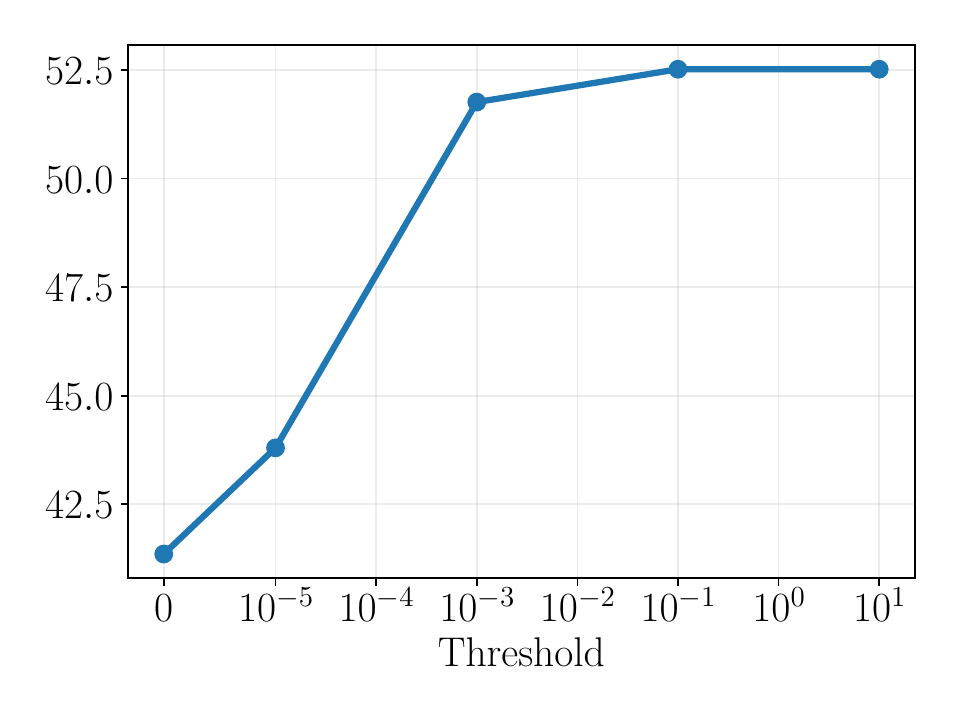}
        \caption{CIFAR100 $\alpha=0.1$}
        % \label{fig:test_accuracy}
    \end{subfigure}
    \begin{subfigure}[b]{0.19\textwidth}
       \centering
       \includegraphics[width=\linewidth]{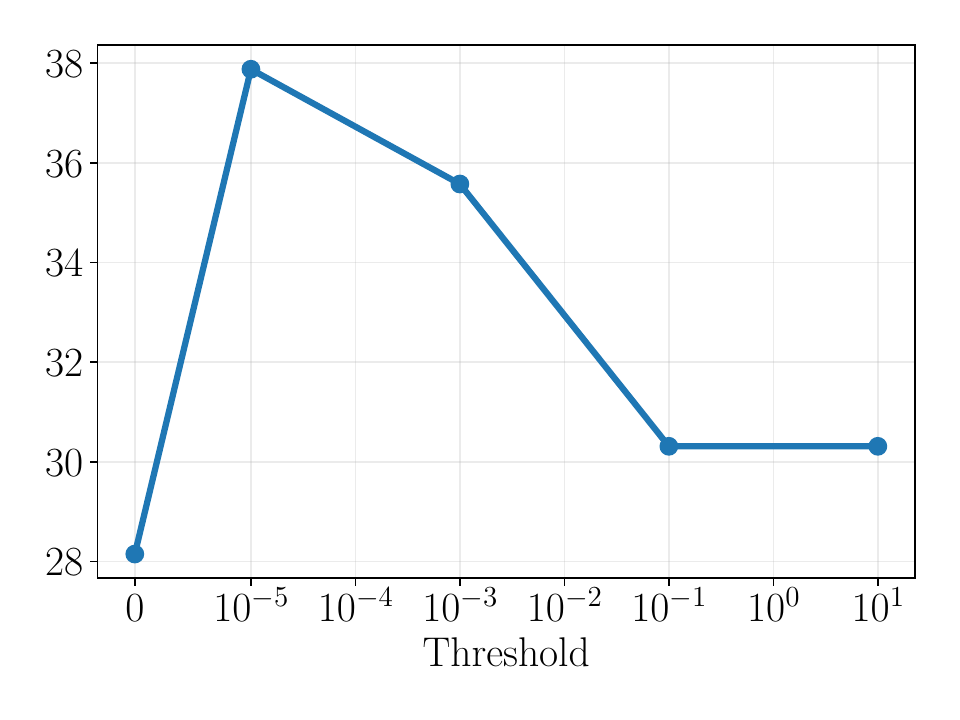}
       \caption{CIFAR100 $\alpha=0.5$}
       % \label{fig:test_accuracy}
    \end{subfigure}
    \caption{
        Validation accuracy for different values of the threshold parameter,
        with individually tuned step sizes. 
    }
    \label{fig:threshold}
\end{figure*}

We study the performance of \textsc{Karula} empirically and compare it against
several existing conventional and personalized FL methods. We evaluate the
performance of the different methods on federated image classification tasks
using CNN models. All experiments were performed on NVIDIA H200 GPUs. 

\paragraph{Methods and hyperparameters.} We compare \textsc{Karula} with nine
FL baselines covering the different approaches to PFL: \textsc{FedAvg}
\citep{McMahan17}, \textsc{pFedMe} \citep{Dinh20}, IFCA \citep{Ghosh20},
\textsc{FedSoft} \citep{Ruan22}, \textsc{pFedGraph} \citep{Ye23},
\textsc{FedRep} \citep{Collins21}, \textsc{pFedHN} \citep{Shamsian21},
\textsc{FedAMP} \citep{Huang21}, and \textsc{Ditto} \citep{Li21}.

To select the hyperparameters for each method, we train with 80\% of the
training data and evaluate on the remaining data. For all methods we sweep the
step size $\eta \in \{2^k \colon k=-10,\dots,-1\}$ and four values of one
extra hyperparameter. The regularization parameters of \textsc{pFedMe},
\textsc{FedSoft}, \textsc{pFedGraph}, \textsc{Ditto}, \textsc{FedAMP} are then
swept over $\{10^{-2}, 10^{-1}, 10^{0}, 10^{1}\}$. For \textsc{Karula}, we
scale the threshold $t = t' \times \sqrt{p} / \max_{i,j} D_{ij}$ and sweep over
$t' \in \{0, 10^{-5}, 10^{-3}, 10^{-1}\}$. We use a normal distribution for the
features and a Bernoulli distribution for the features in the reference
dataset. For IFCA, the number of clusters is swept over $\{2,3,5,8\}$. All
methods use one local epoch, and we use the same global and local step size for
all methods that allow for separate ones.

Details on the selected hyperparameters and the CNN architectures used in the
experiments are given in Appendix~\ref{sec:experimental-details}.

\begin{figure}[htbp]
    \centering
    \begin{subfigure}[b]{0.45\linewidth}
        \centering
        \includegraphics[width=\linewidth]{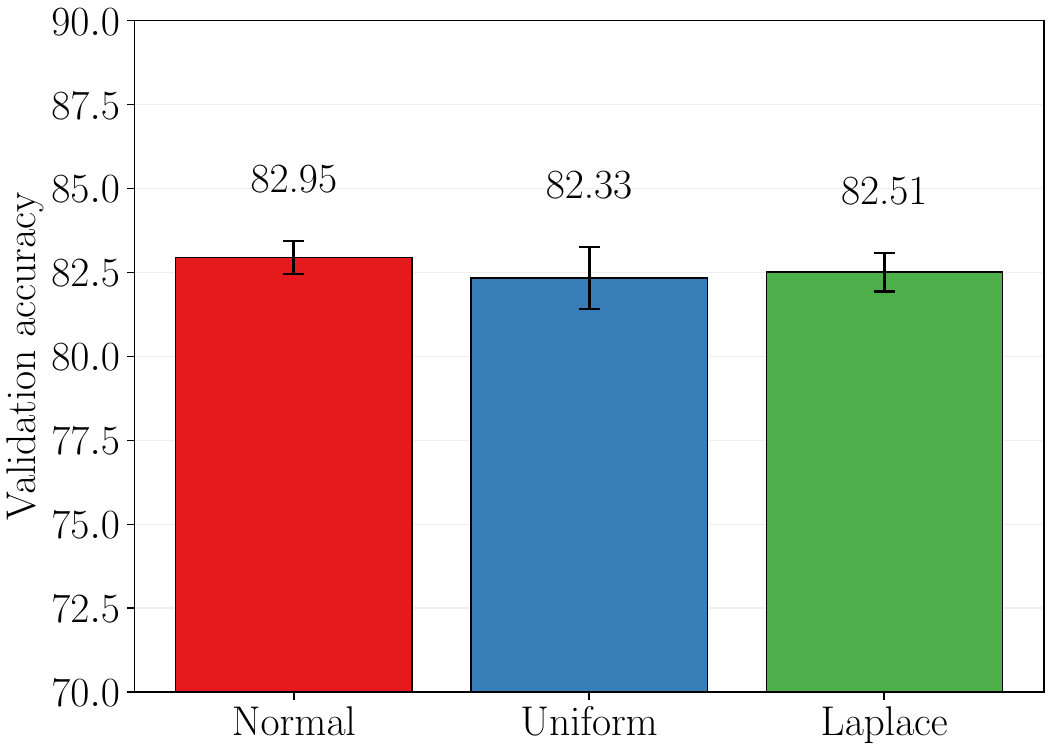}
        \caption{Accuracy vs.\ reference dataset distribution}
        \label{fig:ref-dist}
    \end{subfigure}
    \begin{subfigure}[b]{0.45\linewidth}
        \centering
        \includegraphics[width=\linewidth]{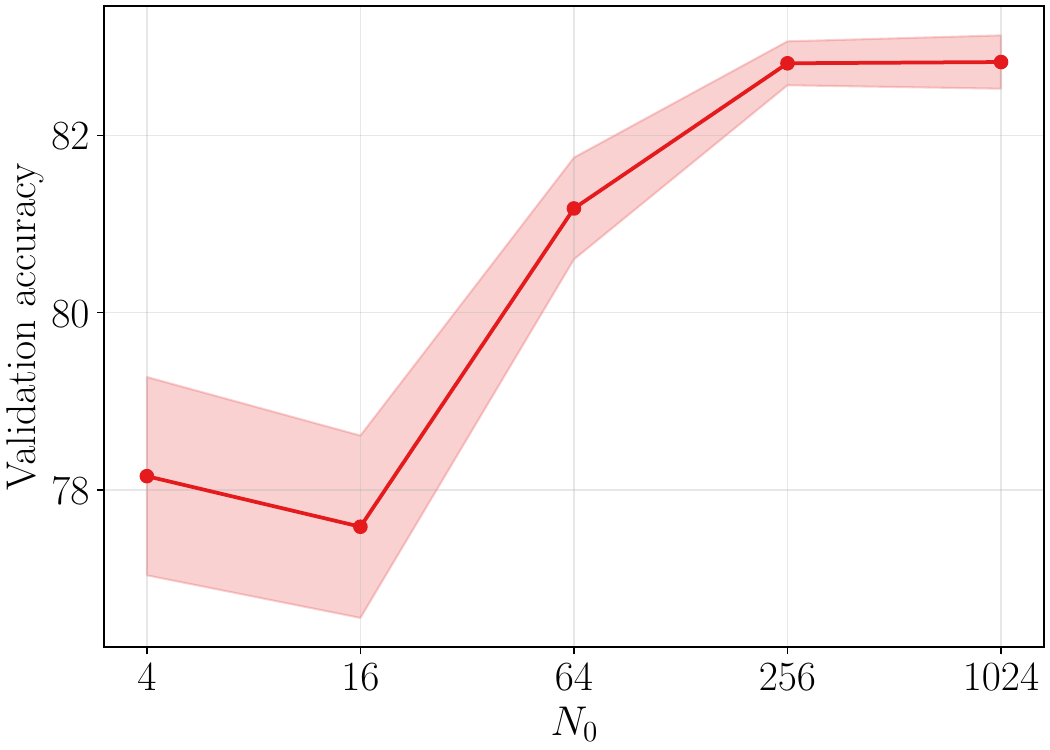}
        \caption{Accuracy vs.\ reference dataset size.}
        \label{fig:ref-size}
    \end{subfigure}
    \caption{
        \textsc{Karula}'s mean validation accuracy (and standard devation error
        bars) on the FEMNIST dataset for (a) different reference distributions
        and (b) different reference dataset sizes $N_0$.
    }
    \label{fig:ref-ablation}
\end{figure}

\paragraph{FEMNIST.} We run experiments on the federated version of the EMNIST
dataset \citep{Cohen17, Caldas19}, where the task is handwritten character
recognition with 62 classes. Each client holds data of handwritten characters
corresponding to one individual writer, so the statistical heterogeneity stems
from both idiosyncratic handwriting and label skewness. We compare the
performances of the different methods on a network of $n=100$ clients. In each
round, a fifth of the clients participate and use a batch size of $64$. 

\paragraph{Label skew CIFAR.} Following the works of \cite{Fallah20, Diao20,
He20, Marfoq21} we run experiments on image classification with label skew
CIFAR datasets \citep{Krizhevsky09}. The samples are distributed among the
clients using a Dirichlet distribution with parameter $\alpha$, where smaller
values increase the heterogeneity. For CIFAR10, we use $n=100$ clients, and for
CIFAR100, we use $n=50$ clients. In each round, a fifth of the clients
participate and use a batch size of 128. We run the experiments for
$\alpha=0.1$ and $\alpha=0.5$ to examine how the methods are affected by the
level of statistical heterogeneity. 

\paragraph{Results.} After hyperparameter selection, we train on all of the
training data and compute average test performance over three random seeds. We
report the main numerical results in Table~\ref{tab:main-results}. Since
fairness is important in this setting, we report the 10th, 50th and 90th
percentile client test accuracies. Here, we see that \textsc{Karula}
outperforms the other methods across 12 out of 15 scenarios and percentiles,
and is always at the top for median client accuracy. 

In Figure~\ref{fig:threshold}, we show the mean client validation accuracies
for several thresholds on each scenario using their respective best learning
rates. Importantly, in four out of five scenarios, the best threshold is
neither the smallest nor the largest of values, suggesting that the favorable
performance specifically comes from \textsc{Karula}'s regularization. In
Figure~\ref{fig:ref-ablation}, we show the results of ablating the reference
data set distribution and size. Figure~\ref{fig:ref-dist} shows that the
distribution of the reference dataset does not produce a statistically
significant change in validation accuracy. Figure~\ref{fig:ref-size} however
shows that performance does degrade if the reference dataset is too small,
further suggesting that the performance of \textsc{Karula} is specifically due
to its data-driven regularization.

In Figures~\ref{fig:femnist-clients} and \ref{fig:cifar-clients}, we examine
what \textsc{Karula}'s distances $D_{ij}$ correspond to qualitatively. We draw
one client at random, and find its nearest and furthest neighbors, respectively
and examine their distributions. With FEMNIST, we can clearly see that the
nearest neighbor has similar handwriting, while the furthest neighbor has
dissimilar handwriting. With CIFAR10 ($\alpha=0.1$) the class density of the
nearest neighbor is very similar while the furthest neighbor has no overlap in
classes.

\begin{figure}[htbp]
    \centering
    \begin{subfigure}[b]{0.3\linewidth}
        \centering
        \includegraphics[width=\linewidth]{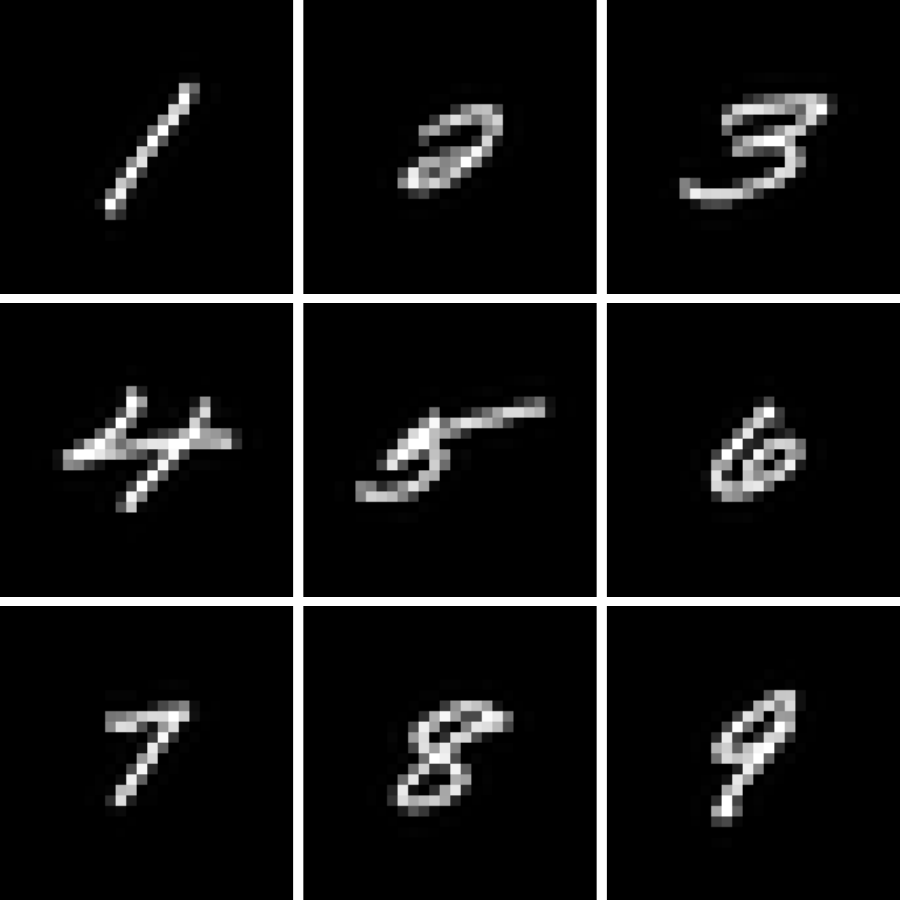}
        \caption{Focal client}
        \label{fig:femnist-focal}
    \end{subfigure}\hfill
    \begin{subfigure}[b]{0.3\linewidth}
        \centering
        \includegraphics[width=\linewidth]{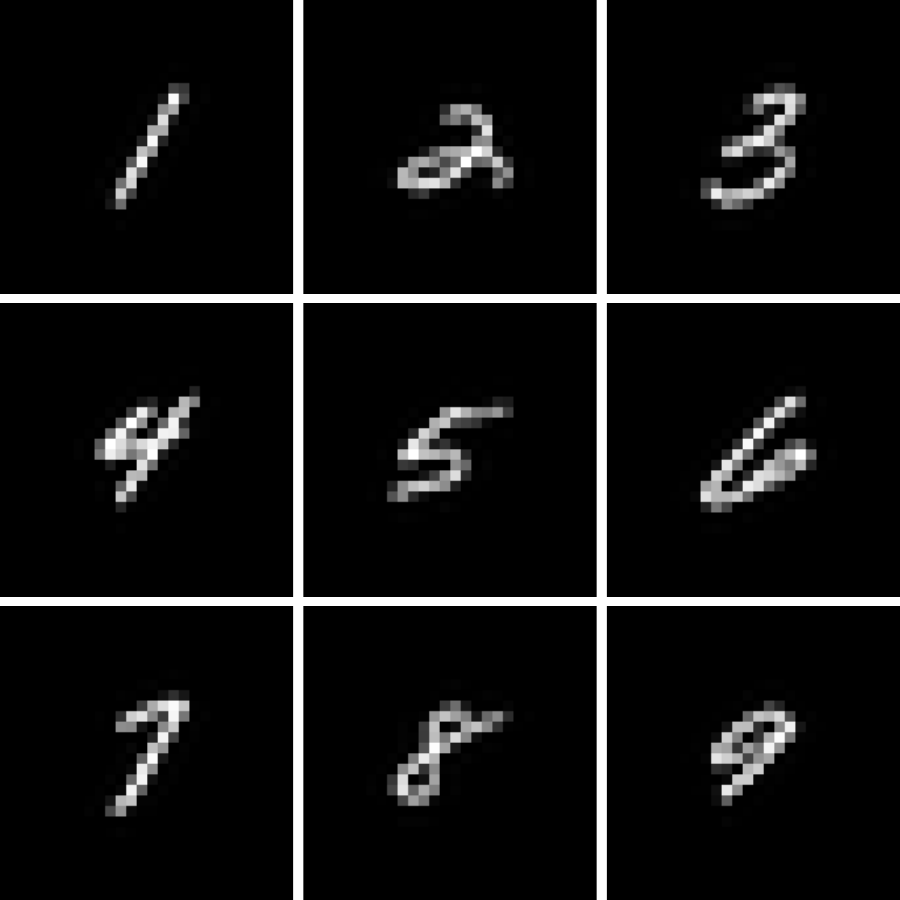}
        \caption{Nearest}
        \label{fig:femnist-nearest}
    \end{subfigure}\hfill
    \begin{subfigure}[b]{0.3\linewidth}
        \centering
        \includegraphics[width=\linewidth]{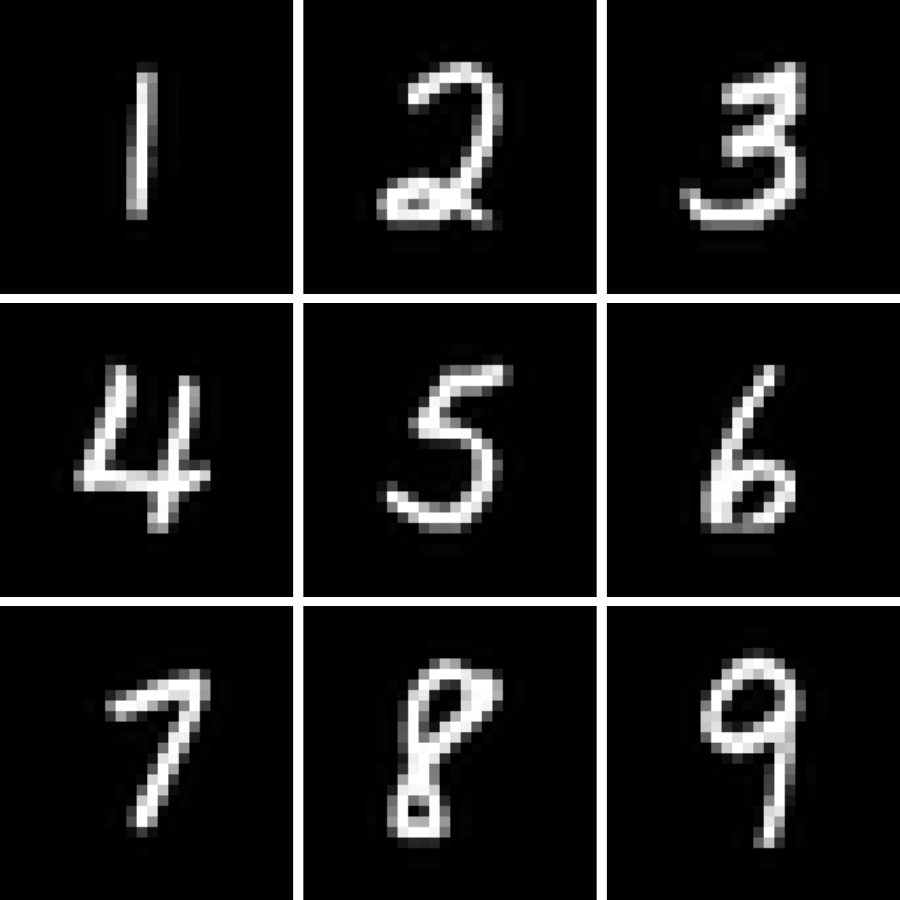}
        \caption{Furthest}
        \label{fig:femnist-furthest}
    \end{subfigure}
    \caption{
        Samples of FEMNIST handwritten digits from (a) client chosen at random,
        (b) its nearest neighbor, and (c) its furthest neighbor.
    }
    \label{fig:femnist-clients}
\end{figure}

\begin{figure}[htbp]
    \centering
    \begin{subfigure}[b]{0.3\linewidth}
        \centering
        \includegraphics[width=\linewidth]{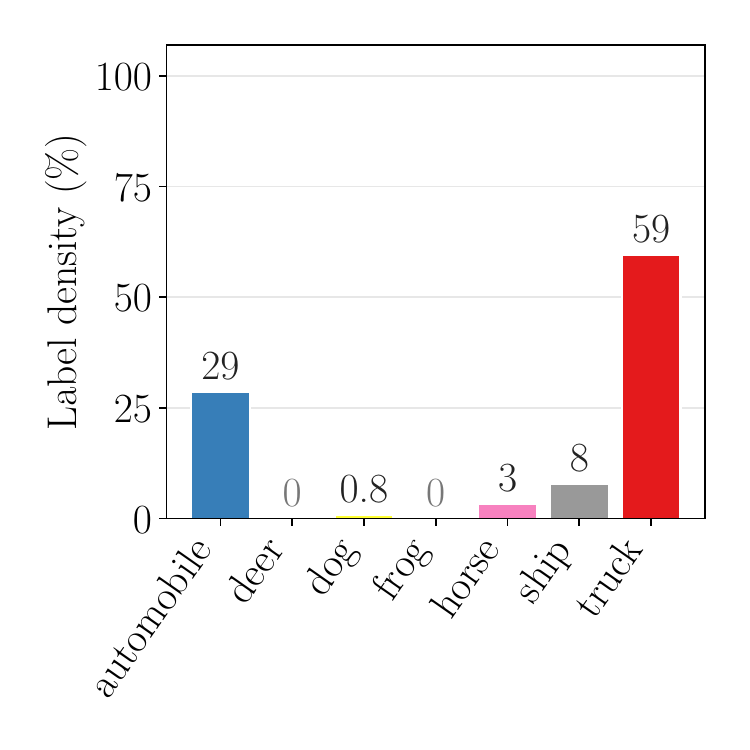}
        \caption{Focal client}
        \label{fig:cifar-focal}
    \end{subfigure}
    \begin{subfigure}[b]{0.3\linewidth}
        \centering
        \includegraphics[width=\linewidth]{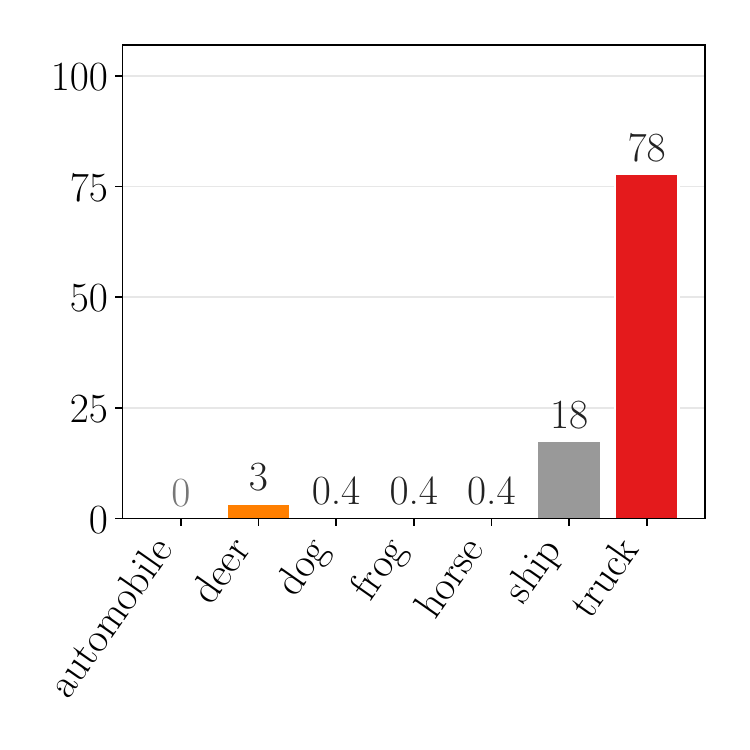}
        \caption{Nearest}
        \label{fig:cifar-nearest}
    \end{subfigure}
    \begin{subfigure}[b]{0.3\linewidth}
        \centering
        \includegraphics[width=\linewidth]{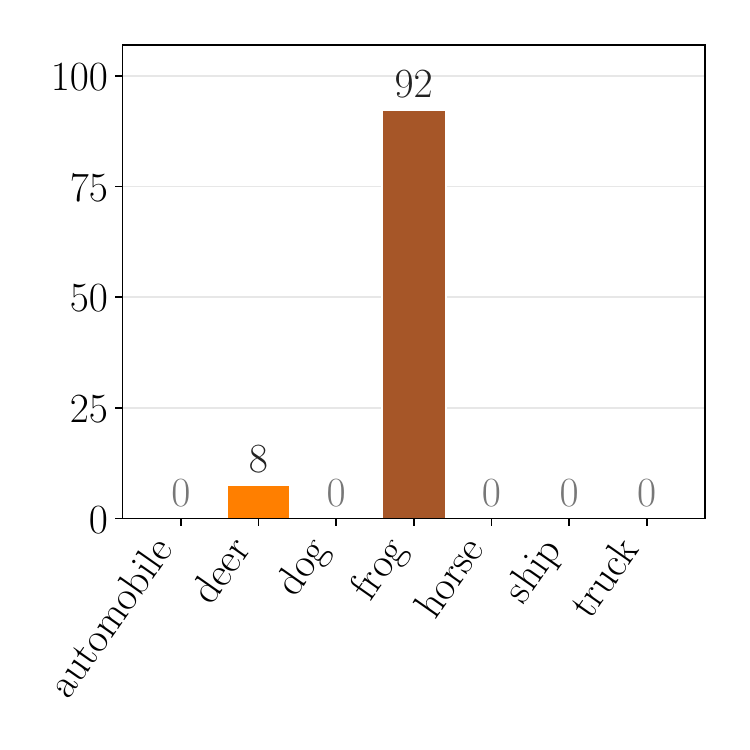}
        \caption{Furthest}
        \label{fig:cifar-furthest}
    \end{subfigure}
    \caption{
        Label densities in CIFAR10 ($\alpha=0.1$) of (a) client chosen at
        random, (b) its nearest neighbor, and (c) its furthest neighbor.
    }
    \label{fig:cifar-clients}
\end{figure}

\section{Discussion}
\label{sec:discussion}

We have introduced \textsc{Karula}, a personalized federated learning method
based on constrained optimization and data-driven measures of client
dissimilarity. The method is motivated by the principle that clients with
similar data distributions should learn similar personalized models. We make
this connection precise by showing that, under mild conditions, the distance
between ideal personalized models can be bounded in terms of the 1-Wasserstein
distance between client distributions. This provides a principled and
task-agnostic justification for using distributional dissimilarities to define
pairwise constraints on model parameters. To make this approach practical in
federated settings, we compute the dissimilarity parameters using linear
Wasserstein embeddings, avoiding both the exchange of raw client data and the
need to solve all pairwise optimal transport problems. We then solve the
resulting constrained optimization problem using a projected stochastic
gradient algorithm that supports partial client participation and incorporates
variance reduction. The method converges to an $\varepsilon$-stationary point
under standard smoothness assumptions, even for nonconvex losses, and our
experiments on standard federated learning benchmarks demonstrate its
effectiveness.

\paragraph{Limitations and extensions.} Although our theoretical and empirical
development focuses on Wasserstein-based dissimilarities, the \textsc{Karula}
optimization framework only requires nonnegative pairwise dissimilarities
$D_{ij}$. More generally, one may embed each client dataset via a map
$E(\mathcal D_i)$ and set $D_{ij} = \|E(\mathcal D_i)-E(\mathcal D_j)\|$, using
the same pairwise constraints. Wasserstein embeddings are therefore one
principled instance of a broader client-embedding perspective. This perspective
also clarifies the role of the Wasserstein guarantee: it is a distributional,
task-agnostic upper bound, not necessarily a tight characterization of model
dissimilarity for every learning problem. In structured settings, task-aware
embeddings may provide sharper notions of personalization similarity. 

Moreover, the proposed method has a server-side memory overhead of $O(n^2)$ in
the number of clients $n$ stemming from the dual variables in the ADMM
projection. In very large scale FL deployments, this may make the default
implementation infeasible in practice. In such cases, we may modify the
projection. For instance, the $D_{ij}$ could be sparsified by thresholding.
Then there would only be as many dual variables as non-zero $D_{ij}$.

\section*{AI Use Statement}

In this work, we used generative AI tools for checking mathematical proofs and
integrating baseline methods into our codebase from public repos. Additionally,
we used these tools to setup and facilitate the numerical experiments presented
in the paper.

\bibliography{references}

@inproceedings{Dinh20,
    author = {Dinh, Canh T and Tran, Nguyen and Nguyen, Josh},
    booktitle = {Advances in Neural Information Processing Systems},
    OPTeditor = {H. Larochelle and M. Ranzato and R. Hadsell and M.F. Balcan and
                 H. Lin},
    pages = {21394--21405},
    OPTpublisher = {Curran Associates, Inc.},
    title = {Personalized Federated Learning with Moreau Envelopes},
    url = {
           https://proceedings.neurips.cc/paper_files/paper/2020/file/f4f1f13c8289ac1b1ee0ff176b56fc60-Paper.pdf
           },
    volume = {33},
    year = {2020},
}

@inproceedings{Fallah20,
    author = {Fallah, Alireza and Mokhtari, Aryan and Ozdaglar, Asuman},
    booktitle = {Advances in Neural Information Processing Systems},
    OPTeditor = {H. Larochelle and M. Ranzato and R. Hadsell and M.F. Balcan and
                 H. Lin},
    pages = {3557--3568},
    OPTpublisher = {Curran Associates, Inc.},
    title = {Personalized Federated Learning with Theoretical Guarantees: A
             Model-Agnostic Meta-Learning Approach},
    url = {
           https://proceedings.neurips.cc/paper_files/paper/2020/file/24389bfe4fe2eba8bf9aa9203a44cdad-Paper.pdf
           },
    volume = {33},
    year = {2020},
}

@inproceedings{McMahan17,
    title = {{Communication-Efficient Learning of Deep Networks from
             Decentralized Data}},
    author = {McMahan, Brendan and Moore, Eider and Ramage, Daniel and Hampson,
              Seth and Arcas, Blaise Aguera},
    booktitle = {Proceedings of the 20th International Conference on Artificial
                 Intelligence and Statistics},
    pages = {1273--1282},
    year = {2017},
    OPTeditor = {Singh, Aarti and Zhu, Jerry},
    volume = {54},
    OPTseries = {Proceedings of Machine Learning Research},
    month = {20--22 Apr},
    OPTpublisher = {PMLR},
    url = {https://proceedings.mlr.press/v54/mcmahan17a.html},
}

@inproceedings{Armacki22,
    title = {Personalized federated learning via convex clustering},
    author = {Armacki, Aleksandar and Bajovic, Dragana and Jakovetic, Dusan and
              Kar, Soummya},
    booktitle = {2022 IEEE International Smart Cities Conference (ISC2)},
    pages = {1--7},
    year = {2022},
    organization = {IEEE},
}

@article{Ghosh20,
    title = {An efficient framework for clustered federated learning},
    author = {Ghosh, Avishek and Chung, Jichan and Yin, Dong and Ramchandran,
              Kannan},
    journal = {Advances in Neural Information Processing Systems},
    volume = {33},
    pages = {19586--19597},
    year = {2020},
}

@article{Deng20,
    title = {Adaptive personalized federated learning},
    author = {Deng, Yuyang and Kamani, Mohammad Mahdi and Mahdavi, Mehrdad},
    journal = {arXiv preprint arXiv:2003.13461},
    year = {2020},
}

@book{Villani08,
    author = {Cédric Villani},
    title = {Optimal Transport},
    publisher = {Springer Berlin, Heidelberg},
    year = {2008},
}

@article{Hard18,
    title = {Federated learning for mobile keyboard prediction},
    author = {Hard, Andrew and Rao, Kanishka and Mathews, Rajiv and Ramaswamy,
              Swaroop and Beaufays, Fran{\c{c}}oise and Augenstein, Sean and
              Eichner, Hubert and Kiddon, Chlo{\'e} and Ramage, Daniel},
    journal = {arXiv preprint arXiv:1811.03604},
    year = {2018},
}

@article{Zhao18,
    title = {Federated learning with non-iid data},
    author = {Zhao, Yue and Li, Meng and Lai, Liangzhen and Suda, Naveen and
              Civin, Damon and Chandra, Vikas},
    journal = {arXiv preprint arXiv:1806.00582},
    year = {2018},
}

@inproceedings{Nguyen22,
    title = {Deep federated learning for autonomous driving},
    author = {Nguyen, Anh and Do, Tuong and Tran, Minh and Nguyen, Binh X and
              Duong, Chien and Phan, Tu and Tjiputra, Erman and Tran, Quang D},
    booktitle = {2022 IEEE Intelligent Vehicles Symposium (IV)},
    pages = {1824--1830},
    year = {2022},
    organization = {IEEE},
}

@article{Qayyum22,
    title = {Collaborative federated learning for healthcare: Multi-modal
             covid-19 diagnosis at the edge},
    author = {Qayyum, Adnan and Ahmad, Kashif and Ahsan, Muhammad Ahtazaz and
              Al-Fuqaha, Ala and Qadir, Junaid},
    journal = {IEEE Open Journal of the Computer Society},
    volume = {3},
    pages = {172--184},
    year = {2022},
    publisher = {IEEE},
}

@article{DworkRoth14,
    title = {The Algorithmic Foundations of Differential Privacy},
    author = {Dwork, Cynthia and Roth, Aaron},
    journal = {Foundations and Trends in Theoretical Computer Science},
    volume = {9},
    number = {3--4},
    pages = {211--407},
    year = {2014},
    doi = {10.1561/0400000042},
    url = {https://www.cis.upenn.edu/~aaroth/Papers/privacybook.pdf},
}

@article{Liu25,
    title = {Wasserstein task embedding for measuring task similarities},
    author = {Liu, Xinran and Bai, Yikun and Lu, Yuzhe and Soltoggio, Andrea and
              Kolouri, Soheil},
    journal = {Neural Networks},
    volume = {181},
    pages = {106796},
    year = {2025},
    publisher = {Elsevier},
}

@article{Defazio14,
    title = {{SAGA: A fast incremental gradient method with support for
             non-strongly convex composite objectives}},
    author = {Defazio, Aaron and Bach, Francis and Lacoste-Julien, Simon},
    journal = {Advances in neural information processing systems},
    volume = {27},
    year = {2014},
}

@inproceedings{Jhunjhunwala22,
    title = {{FedVARP}: Tackling the variance due to partial client
             participation in federated learning},
    author = {Jhunjhunwala, Divyansh and Sharma, Pranay and Nagarkatti, Aushim
              and Joshi, Gauri},
    booktitle = {Uncertainty in Artificial Intelligence},
    pages = {906--916},
    year = {2022},
    organization = {PMLR},
}

@article{Kolouri20,
    title = {Wasserstein embedding for graph learning},
    author = {Kolouri, Soheil and Naderializadeh, Navid and Rohde, Gustavo K and
              Hoffmann, Heiko},
    journal = {arXiv preprint arXiv:2006.09430},
    year = {2020},
}

@article{Boyd11,
    title = {Distributed optimization and statistical learning via the
             alternating direction method of multipliers},
    author = {Boyd, Stephen and Parikh, Neal and Chu, Eric and Peleato, Borja
              and Eckstein, Jonathan and others},
    journal = {Foundations and Trends{\textregistered} in Machine learning},
    volume = {3},
    number = {1},
    pages = {1--122},
    year = {2011},
    publisher = {Now Publishers, Inc.},
}

@inproceedings{Li21,
    title = {Ditto: Fair and robust federated learning through personalization},
    author = {Li, Tian and Hu, Shengyuan and Beirami, Ahmad and Smith, Virginia},
    booktitle = {International conference on machine learning},
    pages = {6357--6368},
    year = {2021},
    organization = {PMLR},
}

@article{Hashemi24,
    title = {{CoBo}: Collaborative Learning via Bilevel Optimization},
    author = {Hashemi, Diba and He, Lie and Jaggi, Martin},
    journal = {arXiv preprint arXiv:2409.05539},
    year = {2024},
}

@article{Dinh22,
    author = {Dinh, Canh T. and Vu, Tung T. and Tran, Nguyen H. and Dao, Minh N.
              and Zhang, Hongyu},
    journal = {IEEE Transactions on Neural Networks and Learning Systems},
    title = {A New Look and Convergence Rate of Federated Multitask Learning
             With {Laplacian} Regularization},
    year = {2024},
    volume = {35},
    number = {6},
    pages = {8075-8085},
}

@article{Zhong18,
    author = {Zhong, Xiaofeng and Guo, Shize and Shan, Hong and Gao, Liang and
              Xue, Di and Zhao, Nan},
    journal = {IEEE Access},
    title = {Feature-Based Transfer Learning Based on Distribution Similarity},
    year = {2018},
    volume = {6},
    number = {},
    pages = {35551-35557},
    doi = {10.1109/ACCESS.2018.2843773},
}

@article{Pan20,
    title = {Transfer learning},
    author = {Pan, Sinno Jialin},
    journal = {Learning},
    volume = {21},
    pages = {1--2},
    year = {2020},
}

@inproceedings{Zhou21,
    title = {Task similarity aware meta learning: Theory-inspired improvement on
             {MAML}},
    author = {Zhou, Pan and Zou, Yingtian and Yuan, Xiao-Tong and Feng, Jiashi
              and Xiong, Caiming and Hoi, Steven},
    booktitle = {Uncertainty in artificial intelligence},
    pages = {23--33},
    year = {2021},
    organization = {PMLR},
}

@article{Shui19,
    title = {A principled approach for learning task similarity in multitask
             learning},
    author = {Shui, Changjian and Abbasi, Mahdieh and Robitaille, Louis-{\'E}
              mile and Wang, Boyu and Gagn{\'e}, Christian},
    journal = {arXiv preprint arXiv:1903.09109},
    year = {2019},
}

@article{Reddi16,
    title = {Proximal stochastic methods for nonsmooth nonconvex finite-sum
             optimization},
    author = {Reddi, Sashank J. and Sra, Suvrit and Poczos, Barnabas and Smola,
              Alexander J},
    journal = {Advances in neural information processing systems},
    volume = {29},
    year = {2016},
}

@article{Kullback51,
    title = {On Information and Sufficiency},
    author = {S. Kullback and R. A. Leibler},
    journal = {Annals of Mathematical Statistics},
    volume = {22},
    number = {1},
    pages = {79--86},
    year = {1951},
    month = mar,
    doi = {10.1214/aoms/1177729694},
}

@article{Lin91,
    title = {Divergence measures based on the Shannon entropy},
    author = {J. Lin},
    journal = {IEEE Transactions on Information Theory},
    volume = {37},
    number = {1},
    pages = {145--151},
    year = {1991},
    doi = {10.1109/18.61115},
}

@article{Kantorovich60,
    title = {Mathematical methods of organizing and planning production},
    author = {Kantorovich, Leonid V},
    journal = {Management science},
    volume = {6},
    number = {4},
    pages = {366--422},
    year = {1960},
    publisher = {INFORMS},
}

@inproceedings{Caldas19,
    title = {{LEAF}: A Benchmark for Federated Settings},
    author = {Caldas, Sebastian and Duddu, Sai Meher Karthik and Wu, Peter and
              Li, Tian and Konečný, Jakub and McMahan, H. Brendan and Smith,
              Virginia and Talwalkar, Ameet},
    booktitle = {Workshop on Federated Learning for Data Privacy and
                 Confidentiality},
    year = {2019},
}

@book{Nesterov18,
    author = {Yurii Nesterov},
    title = {Lectures on Convex Optimization},
    series = {Springer Optimization and Its Applications},
    volume = {137},
    edition = {2},
    year = {2018},
    publisher = {Springer, Cham},
    isbn = {978-3-319-91577-7},
    doi = {10.1007/978-3-319-91578-4},
    url = {https://doi.org/10.1007/978-3-319-91578-4},
    pages = {XXIII, 589},
    note = {eBook ISBN: 978-3-319-91578-4},
    issn = {1931-6828},
    eissn = {1931-6836},
}

@article{Arivazhagan19,
    title = {Federated learning with personalization layers},
    author = {Arivazhagan, Manoj Ghuhan and Aggarwal, Vinay and Singh, Aaditya
              Kumar and Choudhary, Sunav},
    journal = {arXiv preprint arXiv:1912.00818},
    year = {2019},
}

@article{Yang23,
    title = {A spectral condition for feature learning},
    author = {Yang, Greg and Simon, James B and Bernstein, Jeremy},
    journal = {arXiv preprint arXiv:2310.17813},
    year = {2023},
}

@inproceedings{Wu23,
    title = {Personalized federated learning under mixture of distributions},
    author = {Wu, Yue and Zhang, Shuaicheng and Yu, Wenchao and Liu, Yanchi and
              Gu, Quanquan and Zhou, Dawei and Chen, Haifeng and Cheng, Wei},
    booktitle = {International Conference on Machine Learning},
    pages = {37860--37879},
    year = {2023},
    organization = {PMLR},
}

@article{Lin24,
    title = {Fedspd: A soft-clustering approach for personalized decentralized
             federated learning},
    author = {Lin, I and Yagan, Osman and Joe-Wong, Carlee and others},
    journal = {arXiv preprint arXiv:2410.18862},
    year = {2024},
}

@article{Marfoq21,
    title = {Federated multi-task learning under a mixture of distributions},
    author = {Marfoq, Othmane and Neglia, Giovanni and Bellet, Aur{\'e}lien and
              Kameni, Laetitia and Vidal, Richard},
    journal = {Advances in neural information processing systems},
    volume = {34},
    OPTpages = {15434--15447},
    year = {2021},
}

@inproceedings{Ruan22,
    title = {{FedSoft}: Soft clustered federated learning with proximal local
             updating},
    author = {Ruan, Yichen and Joe-Wong, Carlee},
    booktitle = {Proceedings of the AAAI conference on artificial intelligence},
    volume = {36},
    number = {7},
    pages = {8124--8131},
    year = {2022},
}

@article{Diao20,
    title = {{HeteroFL}: Computation and communication efficient federated
             learning for heterogeneous clients},
    author = {Diao, Enmao and Ding, Jie and Tarokh, Vahid},
    journal = {arXiv preprint arXiv:2010.01264},
    year = {2020},
}

@article{He20,
    title = {Fedml: A research library and benchmark for federated machine
             learning},
    author = {He, Chaoyang and Li, Songze and So, Jinhyun and Zeng, Xiao and
              Zhang, Mi and Wang, Hongyi and Wang, Xiaoyang and Vepakomma,
              Praneeth and Singh, Abhishek and Qiu, Hang and others},
    journal = {arXiv preprint arXiv:2007.13518},
    year = {2020},
}

@article{Cohen17,
    title = {{EMNIST}: an extension of {MNIST} to handwritten letters},
    author = {Cohen, Gregory and Afshar, Saeed and Tapson, Jonathan and van
              Schaik, Andre},
    journal = {2017 International Joint Conference on Neural Networks (IJCNN)},
    year = {2017},
    doi = {10.1109/IJCNN.2017.7966217},
}

@techreport{Krizhevsky09,
    title = {Learning Multiple Layers of Features from Tiny Images},
    author = {Krizhevsky, Alex},
    institution = {University of Toronto},
    year = {2009},
}

@article{Schmidt11,
    title = {Convergence rates of inexact proximal-gradient methods for convex
             optimization},
    author = {Schmidt, Mark and Roux, Nicolas and Bach, Francis},
    journal = {Advances in neural information processing systems},
    volume = {24},
    year = {2011},
}

@inproceedings{Chen22,
    title = {On Bridging Generic and Personalized Federated Learning for Image
             Classification},
    author = {Hong-You Chen and Wei-Lun Chao},
    booktitle = {International Conference on Learning Representations},
    year = {2022},
    url = {https://openreview.net/forum?id=I1hQbx10Kxn},
}

@inproceedings{Shamsian21,
    title = {Personalized Federated Learning using Hypernetworks},
    author = {Shamsian, Aviv and Navon, Aviv and Fetaya, Ethan and Chechik, Gal},
    booktitle = {Proceedings of the 38th International Conference on Machine
                 Learning},
    pages = {9489--9502},
    year = {2021},
    editor = {Meila, Marina and Zhang, Tong},
    volume = {139},
    series = {Proceedings of Machine Learning Research},
    month = {18--24 Jul},
    publisher = {PMLR},
    url = {https://proceedings.mlr.press/v139/shamsian21a.html},
}

@inproceedings{Collins21,
    title = {Exploiting Shared Representations for Personalized Federated
             Learning},
    author = {Collins, Liam and Hassani, Hamed and Mokhtari, Aryan and
              Shakkottai, Sanjay},
    booktitle = {Proceedings of the 38th International Conference on Machine
                 Learning},
    pages = {2089--2099},
    year = {2021},
    editor = {Meila, Marina and Zhang, Tong},
    volume = {139},
    series = {Proceedings of Machine Learning Research},
    month = {18--24 Jul},
    publisher = {PMLR},
    url = {https://proceedings.mlr.press/v139/collins21a.html},
}

@inproceedings{Ye23,
    title = {Personalized Federated Learning with Inferred Collaboration Graphs},
    author = {Ye, Rui and Ni, Zhenyang and Wu, Fangzhao and Chen, Siheng and
              Wang, Yanfeng},
    booktitle = {Proceedings of the 40th International Conference on Machine
                 Learning},
    pages = {39801--39817},
    year = {2023},
    editor = {Krause, Andreas and Brunskill, Emma and Cho, Kyunghyun and
              Engelhardt, Barbara and Sabato, Sivan and Scarlett, Jonathan},
    volume = {202},
    series = {Proceedings of Machine Learning Research},
    month = {23--29 Jul},
    publisher = {PMLR},
    url = {https://proceedings.mlr.press/v202/ye23b.html},
}

@inproceedings{Huang21,
    title = {Personalized cross-silo federated learning on non-iid data},
    author = {Huang, Yutao and Chu, Lingyang and Zhou, Zirui and Wang, Lanjun
              and Liu, Jiangchuan and Pei, Jian and Zhang, Yong},
    booktitle = {Proceedings of the AAAI conference on artificial intelligence},
    volume = {35},
    number = {9},
    pages = {7865--7873},
    year = {2021},
}

@inproceedings{McMahan18,
    title = {Learning Differentially Private Recurrent Language Models},
    author = {Brendan McMahan and Daniel Ramage and Kunal Talwar and Li Zhang},
    year = {2018},
    URL = {https://openreview.net/pdf?id=BJ0hF1Z0b},
    booktitle = {International Conference on Learning Representations (ICLR)},
}

%%%%%%%%%%%%%%%%%%%%%%%%%%%%%%%%%%%%%%%%%%%%%%%%%%%%%%%%%%%%
\section*{Checklist}

\begin{enumerate}

  \item For all models and algorithms presented, check if you include:
  \begin{enumerate}
      \item A clear description of the mathematical setting, assumptions, algorithm, and/or model. Yes. Sections~\ref{sec:method} and \ref{sec:analysis}.
      \item An analysis of the properties and complexity (time, space, sample size) of any algorithm. Yes. Section~\ref{ssec:algorithm}.
    \item (Optional) Anonymized source code, with specification of all dependencies, including external libraries. No. We commit to including all code necessary to reproduce experiments upon acceptance.
  \end{enumerate}

  \item For any theoretical claim, check if you include:
  \begin{enumerate}
      \item Statements of the full set of assumptions of all theoretical results. Yes. Section~\ref{sec:analysis}.
      \item Complete proofs of all theoretical results. Yes. Appendix~\ref{sec:proof}.
      \item Clear explanations of any assumptions. Yes. Sections~\ref{ssec:main-idea} and \ref{sec:analysis}.
  \end{enumerate}

  \item For all figures and tables that present empirical results, check if you include:
  \begin{enumerate}
    \item The code, data, and instructions needed to reproduce the main experimental results (either in the supplemental material or as a URL). No. We commit to including all code necessary to reproduce experiments upon acceptance.
    \item All the training details (e.g., data splits, hyperparameters, how they were chosen). Yes. Training details are stated in Section~\ref{sec:experiments} and Appendix~\ref{sec:experimental-details}.
    \item A clear definition of the specific measure or statistics and errorbars (e.g., with respect to the random seed after running experiments multiple times). Yes. Section~\ref{sec:experiments}.
    \item A description of the computing infrastructure used. (e.g., type of GPUs, internal cluster, or cloud provider). Yes. Section~\ref{sec:experiments}.
  \end{enumerate}

  \item If you are using existing assets (e.g., code, data, models) or curating/releasing new assets, check if you include:
  \begin{enumerate}
      \item Citations of the creator If your work uses existing assets. Yes. Datasets and methods used are properly cited in Section~\ref{sec:experiments}.
    \item The license information of the assets, if applicable. Not applicable.
    \item New assets either in the supplemental material or as a URL, if applicable. Not Applicable.
    \item Information about consent from data providers/curators. Not Applicable.
    \item Discussion of sensible content if applicable, e.g., personally identifiable information or offensive content. Not Applicable.
  \end{enumerate}

  \item If you used crowdsourcing or conducted research with human subjects, check if you include:
  \begin{enumerate}
    \item The full text of instructions given to participants and screenshots. Not Applicable.
    \item Descriptions of potential participant risks, with links to Institutional Review Board (IRB) approvals if applicable. Not Applicable.
    \item The estimated hourly wage paid to participants and the total amount spent on participant compensation. Not Applicable.
  \end{enumerate}

\end{enumerate}

\clearpage
\appendix
\thispagestyle{empty}

% Supplementary material: To improve readability, you must use a single-column format for the supplementary material.
\onecolumn
\aistatstitle{Supplementary material}

\section{Experimental details}
\label{sec:experimental-details}

\subsection{Complete results} 
We provide the complete main experimental results in
Table~\ref{tab:complete-results}.

% Required packages: booktabs, multirow, graphicx, float
\begin{table}[H]
\centering
\caption{
    Average 10th, 50th, and 90th percentile client test accuracies on FEMNIST
    and CIFAR datasets over three random seeds with $\pm$ standard deviations.
}
\label{tab:complete-results}
\resizebox{\textwidth}{!}{%
\begin{tabular}{llccccc}
\toprule
\textbf{Method} & \textbf{Percentile}
& \textbf{FEMNIST}
& \textbf{CIFAR10 $\alpha=0.1$}
& \textbf{CIFAR10 $\alpha=0.5$}
& \textbf{CIFAR100 $\alpha=0.1$}
& \textbf{CIFAR100 $\alpha=0.5$} \\
\midrule

\multirow{3}{*}{\textsc{FedAvg}}
& 10th & $71.61 \pm 2.22\%$ & $44.44 \pm 16.84\%$ & $59.61 \pm 2.42\%$ & $31.48 \pm 2.19\%$ & $31.65 \pm 1.34\%$ \\
& 50th & $83.07 \pm 0.85\%$ & $63.78 \pm 8.69\%$ & $70.45 \pm 1.20\%$ & $38.90 \pm 0.95\%$ & $37.40 \pm 3.46\%$ \\
& 90th & $91.83 \pm 0.35\%$ & $78.72 \pm 1.92\%$ & $79.02 \pm 1.28\%$ & $46.83 \pm 1.09\%$ & $42.37 \pm 3.82\%$ \\
\midrule

\multirow{3}{*}{\textsc{pFedMe}}
& 10th & $46.69 \pm 5.01\%$ & $59.65 \pm 4.39\%$ & $43.10 \pm 1.28\%$ & $30.81 \pm 0.78\%$ & $25.92 \pm 1.93\%$ \\
& 50th & $60.80 \pm 5.13\%$ & $81.78 \pm 2.21\%$ & $58.47 \pm 1.84\%$ & $43.38 \pm 1.06\%$ & $32.41 \pm 0.46\%$ \\
& 90th & $74.92 \pm 5.24\%$ & $98.11 \pm 0.74\%$ & $73.32 \pm 2.77\%$ & $51.50 \pm 2.70\%$ & $39.28 \pm 1.87\%$ \\
\midrule

\multirow{3}{*}{IFCA}
& 10th & $71.08 \pm 2.08\%$ & $46.63 \pm 10.03\%$ & $\mathbf{60.08 \pm 0.78\%}$ & $31.34 \pm 2.88\%$ & $32.43 \pm 2.23\%$ \\
& 50th & $82.18 \pm 0.22\%$ & $65.67 \pm 4.15\%$ & $69.17 \pm 1.97\%$ & $39.35 \pm 1.37\%$ & $36.17 \pm 2.09\%$ \\
& 90th & $92.12 \pm 0.72\%$ & $85.07 \pm 1.26\%$ & $76.61 \pm 1.49\%$ & $48.49 \pm 2.70\%$ & $41.67 \pm 3.85\%$ \\
\midrule

\multirow{3}{*}{\textsc{FedSoft}}
& 10th & $40.61 \pm 10.85\%$ & $59.42 \pm 5.85\%$ & $43.09 \pm 3.44\%$ & $39.80 \pm 2.72\%$ & $12.92 \pm 0.91\%$ \\
& 50th & $63.47 \pm 1.97\%$ & $82.85 \pm 0.83\%$ & $60.74 \pm 0.36\%$ & $49.49 \pm 0.59\%$ & $36.96 \pm 1.13\%$ \\
& 90th & $76.33 \pm 1.19\%$ & $96.84 \pm 0.84\%$ & $73.47 \pm 2.14\%$ & $57.84 \pm 1.10\%$ & $\mathbf{45.27 \pm 1.63\%}$ \\
\midrule

\multirow{3}{*}{\textsc{pFedGraph}}
& 10th & $64.65 \pm 4.24\%$ & $21.01 \pm 3.36\%$ & $41.85 \pm 2.51\%$ & $23.71 \pm 2.93\%$ & $29.31 \pm 1.21\%$ \\
& 50th & $80.66 \pm 2.51\%$ & $49.75 \pm 2.13\%$ & $57.59 \pm 0.45\%$ & $31.80 \pm 2.66\%$ & $36.19 \pm 0.27\%$ \\
& 90th & $88.78 \pm 2.31\%$ & $71.27 \pm 1.97\%$ & $69.43 \pm 0.26\%$ & $41.20 \pm 2.47\%$ & $42.61 \pm 0.64\%$ \\
\midrule

\multirow{3}{*}{\textsc{FedRep}}
& 10th & $53.49 \pm 1.25\%$ & $\mathbf{69.72 \pm 2.80\%}$ & $56.04 \pm 1.47\%$ & $42.74 \pm 2.36\%$ & $25.49 \pm 2.14\%$ \\
& 50th & $65.74 \pm 1.60\%$ & $86.37 \pm 1.03\%$ & $66.32 \pm 0.58\%$ & $50.66 \pm 2.50\%$ & $30.68 \pm 1.63\%$ \\
& 90th & $76.24 \pm 0.75\%$ & $97.69 \pm 1.10\%$ & $76.91 \pm 1.03\%$ & $57.89 \pm 2.63\%$ & $35.44 \pm 2.20\%$ \\
\midrule

\multirow{3}{*}{\textsc{pFedHN}}
& 10th & $66.12 \pm 1.88\%$ & $39.76 \pm 3.50\%$ & $6.51 \pm 3.71\%$ & $30.76 \pm 3.33\%$ & $17.53 \pm 2.85\%$ \\
& 50th & $79.65 \pm 2.12\%$ & $67.81 \pm 2.36\%$ & $26.87 \pm 6.88\%$ & $43.97 \pm 1.18\%$ & $31.65 \pm 0.78\%$ \\
& 90th & $89.23 \pm 0.46\%$ & $93.67 \pm 1.34\%$ & $48.04 \pm 4.48\%$ & $54.22 \pm 2.92\%$ & $39.13 \pm 2.59\%$ \\
\midrule

\multirow{3}{*}{\textsc{FedAMP}}
& 10th & $50.97 \pm 4.27\%$ & $63.50 \pm 6.16\%$ & $43.42 \pm 3.90\%$ & $40.49 \pm 0.60\%$ & $22.59 \pm 1.67\%$ \\
& 50th & $66.41 \pm 4.27\%$ & $82.51 \pm 1.21\%$ & $61.28 \pm 0.67\%$ & $48.23 \pm 1.99\%$ & $36.62 \pm 2.01\%$ \\
& 90th & $79.87 \pm 4.38\%$ & $96.72 \pm 1.96\%$ & $73.81 \pm 1.49\%$ & $57.96 \pm 1.84\%$ & $44.67 \pm 1.65\%$ \\
\midrule

\multirow{3}{*}{\textsc{Ditto}}
& 10th & $58.58 \pm 2.45\%$ & $64.68 \pm 4.63\%$ & $47.46 \pm 5.70\%$ & $33.87 \pm 2.70\%$ & $24.07 \pm 0.58\%$ \\
& 50th & $73.25 \pm 0.94\%$ & $85.28 \pm 2.42\%$ & $70.23 \pm 0.32\%$ & $49.50 \pm 1.15\%$ & $38.12 \pm 0.41\%$ \\
& 90th & $84.39 \pm 1.67\%$ & $97.90 \pm 1.02\%$ & $80.62 \pm 0.95\%$ & $60.24 \pm 1.79\%$ & $44.82 \pm 1.11\%$ \\
\midrule

\multirow{3}{*}{\textsc{Karula}}
& 10th & $\mathbf{73.53 \pm 1.14\%}$ & $63.23 \pm 3.01\%$ & $58.84 \pm 2.92\%$ & $\mathbf{45.22 \pm 2.02\%}$ & $\mathbf{33.82 \pm 0.81\%}$ \\
& 50th & $\mathbf{84.02 \pm 0.71\%}$ & $\mathbf{86.68 \pm 0.59\%}$ & $\mathbf{72.60 \pm 1.51\%}$ & $\mathbf{53.64 \pm 1.09\%}$ & $\mathbf{38.44 \pm 0.32\%}$ \\
& 90th & $\mathbf{92.52 \pm 1.23\%}$ & $\mathbf{99.05 \pm 1.07\%}$ & $\mathbf{82.56 \pm 0.71\%}$ & $\mathbf{61.39 \pm 1.22\%}$ & $43.06 \pm 0.95\%$ \\

\bottomrule
\end{tabular}%
}
\end{table}

\subsection{Hyperparameters and architectures} 

We report all hyperparameters used for the main results in
Tables~\ref{tab:hyperparameters-femnist}-\ref{tab:hyperparameters-cifar100-alpha-05},
and the architectures in Tables~\ref{tab:cnn-femnist} and \ref{tab:cnn-cifar}.

\begin{table}[p]
\centering
\caption{Method hyperparameters for FEMNIST.}
\label{tab:hyperparameters-femnist}
\small
\setlength{\tabcolsep}{4pt}
\renewcommand{\arraystretch}{1.15}
\begin{tabularx}{\textwidth}{@{}lrl>{\raggedright\arraybackslash}X@{}}
\toprule
\textbf{Method} & \shortstack[r]{\textbf{Learning}\\\textbf{rate}} & \shortstack[l]{\textbf{Tuned}\\\textbf{hyperparameter}} & \textbf{Other hyperparameters} \\
\midrule
\textsc{FedAvg} & 0.25 & \textemdash{} & \texttt{n\_\allowbreak{}local\_\allowbreak{}epochs=\allowbreak{}1} \\
\addlinespace[3pt]
\textsc{pFedMe} & 0.0625 & \texttt{moreau\_\allowbreak{}param=\allowbreak{}10.0} & \texttt{beta=\allowbreak{}1.0}; \texttt{inner\_\allowbreak{}steps=\allowbreak{}5}; \texttt{mu=\allowbreak{}0.001}; \texttt{n\_\allowbreak{}local\_\allowbreak{}epochs=\allowbreak{}1} \\
\addlinespace[3pt]
\textsc{IFCA} & 0.125 & \texttt{k\_\allowbreak{}clusters=\allowbreak{}3} & \texttt{aggregation=\allowbreak{}samples}; \texttt{max\_\allowbreak{}init\_\allowbreak{}attempts=\allowbreak{}100}; \texttt{min\_\allowbreak{}cluster\_\allowbreak{}fraction=\allowbreak{}0.0}; \texttt{n\_\allowbreak{}local\_\allowbreak{}epochs=\allowbreak{}1}; \texttt{variant=\allowbreak{}full} \\
\addlinespace[3pt]
\textsc{FedSoft} & 0.25 & \texttt{reg\_\allowbreak{}weight=\allowbreak{}0.01} & \texttt{k\_\allowbreak{}clusters=\allowbreak{}4}; \texttt{n\_\allowbreak{}local\_\allowbreak{}epochs=\allowbreak{}1}; \texttt{refresh\_\allowbreak{}every=\allowbreak{}2}; \texttt{smoothing=\allowbreak{}0.0001} \\
\addlinespace[3pt]
\textsc{pFedGraph} & 0.25 & \texttt{reg\_\allowbreak{}param=\allowbreak{}0.1} & \texttt{alpha=\allowbreak{}0.8}; \texttt{n\_\allowbreak{}local\_\allowbreak{}epochs=\allowbreak{}1} \\
\addlinespace[3pt]
{\textsc{FedAMP}} & 0.125 & {\texttt{lambda\_\allowbreak{}1=\allowbreak{}1.0}} & {\texttt{momentum=\allowbreak{}0.9}; \texttt{n\_\allowbreak{}local\_\allowbreak{}epochs=\allowbreak{}1}; \texttt{self\_\allowbreak{}weight=\allowbreak{}0.3}; \texttt{similarity\_\allowbreak{}scale=\allowbreak{}10.0}; \texttt{weight\_\allowbreak{}decay=\allowbreak{}1e-05}} \\
\addlinespace[3pt]
{\textsc{Ditto}} & 0.125 & \texttt{lam=\allowbreak{}0.1} & \texttt{n\_\allowbreak{}local\_\allowbreak{}epochs=\allowbreak{}1}; \texttt{sampling=\allowbreak{}2} \\
\addlinespace[3pt]
\textsc{Karula} & 0.25 & \texttt{threshold=\allowbreak{}0.001} & \texttt{cost\_\allowbreak{}fn=\allowbreak{}torch.cdist (Euclidean)}; \texttt{initialize\_\allowbreak{}memory=\allowbreak{}all clients before round 1}; \texttt{n\_\allowbreak{}local\_\allowbreak{}epochs=\allowbreak{}1}; \texttt{n\_\allowbreak{}ref=\allowbreak{}1000}; \texttt{projection\_\allowbreak{}max\_\allowbreak{}iter=\allowbreak{}5000}; \texttt{projection\_\allowbreak{}tolerance=\allowbreak{}0.0001} \\
\addlinespace[3pt]
\textsc{FedRep} & 0.125 & \textemdash{} & \texttt{momentum=\allowbreak{}0.0}; \texttt{n\_\allowbreak{}head\_\allowbreak{}epochs=\allowbreak{}1}; \texttt{n\_\allowbreak{}local\_\allowbreak{}epochs=\allowbreak{}1};\newline \texttt{weight\_\allowbreak{}decay=\allowbreak{}0.0} \\
\addlinespace[3pt]
\textsc{pFedHN} & 0.125 & \textemdash{} & \texttt{embed\_\allowbreak{}dim=\allowbreak{}20}; \texttt{hyper\_\allowbreak{}hid=\allowbreak{}100}; \texttt{hyper\_\allowbreak{}weight\_\allowbreak{}decay=\allowbreak{}0.0};\newline \texttt{hypernetwork\_\allowbreak{}gradient\_\allowbreak{}clip\_\allowbreak{}norm=\allowbreak{}50}; \texttt{local\_\allowbreak{}gradient\_\allowbreak{}clip\_\allowbreak{}norm=\allowbreak{}50}; \texttt{momentum=\allowbreak{}0.0};\newline \texttt{n\_\allowbreak{}hidden\_\allowbreak{}layer=\allowbreak{}3}; \texttt{n\_\allowbreak{}local\_\allowbreak{}epochs=\allowbreak{}1}; \texttt{weight\_\allowbreak{}decay=\allowbreak{}0.0} \\
\bottomrule
\end{tabularx}
\end{table}

\begin{table}[p]
\centering
\caption{Method hyperparameters for CIFAR10 $\alpha=0.1$.}
\label{tab:hyperparameters-cifar10-alpha-01}
\small
\setlength{\tabcolsep}{4pt}
\renewcommand{\arraystretch}{1.15}
\begin{tabularx}{\textwidth}{@{}lrl>{\raggedright\arraybackslash}X@{}}
\toprule
\textbf{Method} & \shortstack[r]{\textbf{Learning}\\\textbf{rate}} & \shortstack[l]{\textbf{Tuned}\\\textbf{hyperparameter}} & \textbf{Other hyperparameters} \\
\midrule
\textsc{FedAvg} & 0.015625 & \textemdash{} & \texttt{n\_\allowbreak{}local\_\allowbreak{}epochs=\allowbreak{}1} \\
\addlinespace[3pt]
\textsc{pFedMe} & 0.015625 & \texttt{moreau\_\allowbreak{}param=\allowbreak{}10.0} & \texttt{beta=\allowbreak{}1.0}; \texttt{inner\_\allowbreak{}steps=\allowbreak{}5}; \texttt{mu=\allowbreak{}0.001}; \texttt{n\_\allowbreak{}local\_\allowbreak{}epochs=\allowbreak{}1} \\
\addlinespace[3pt]
\textsc{IFCA} & 0.0078125 & \texttt{k\_\allowbreak{}clusters=\allowbreak{}3} & \texttt{aggregation=\allowbreak{}uniform}; \texttt{max\_\allowbreak{}init\_\allowbreak{}attempts=\allowbreak{}100}; \texttt{min\_\allowbreak{}cluster\_\allowbreak{}fraction=\allowbreak{}0.0}; \texttt{n\_\allowbreak{}local\_\allowbreak{}epochs=\allowbreak{}1}; \texttt{variant=\allowbreak{}full} \\
\addlinespace[3pt]
\textsc{FedSoft} & 0.0078125 & \texttt{reg\_\allowbreak{}weight=\allowbreak{}1.0} & \texttt{k\_\allowbreak{}clusters=\allowbreak{}4}; \texttt{n\_\allowbreak{}local\_\allowbreak{}epochs=\allowbreak{}1}; \texttt{refresh\_\allowbreak{}every=\allowbreak{}2}; \texttt{smoothing=\allowbreak{}0.0001} \\
\addlinespace[3pt]
\textsc{pFedGraph} & 0.03125 & \texttt{reg\_\allowbreak{}param=\allowbreak{}1.0} & \texttt{alpha=\allowbreak{}0.8}; \texttt{n\_\allowbreak{}local\_\allowbreak{}epochs=\allowbreak{}1} \\
\addlinespace[3pt]
\textsc{FedAMP} & 0.015625 & {\texttt{lambda\_\allowbreak{}1=\allowbreak{}0.01}} & {\texttt{momentum=\allowbreak{}0.9}; \texttt{n\_\allowbreak{}local\_\allowbreak{}epochs=\allowbreak{}1}; \texttt{self\_\allowbreak{}weight=\allowbreak{}0.3}; \texttt{similarity\_\allowbreak{}scale=\allowbreak{}10.0}; \texttt{weight\_\allowbreak{}decay=\allowbreak{}1e-05}} \\
\addlinespace[3pt]
{\textsc{Ditto}} & {0.03125} & {\texttt{lam=\allowbreak{}1.0}} & {\texttt{n\_\allowbreak{}local\_\allowbreak{}epochs=\allowbreak{}1}; \texttt{sampling=\allowbreak{}2}} \\
\addlinespace[3pt]
\textsc{Karula} & 0.015625 & \texttt{threshold=\allowbreak{}0.001} & \texttt{cost\_\allowbreak{}fn=\allowbreak{}torch.cdist (Euclidean)}; \texttt{initialize\_\allowbreak{}memory=\allowbreak{}all clients before round 1}; \texttt{n\_\allowbreak{}local\_\allowbreak{}epochs=\allowbreak{}1}; \texttt{n\_\allowbreak{}ref=\allowbreak{}1000}; \texttt{projection\_\allowbreak{}max\_\allowbreak{}iter=\allowbreak{}5000}; \texttt{projection\_\allowbreak{}tolerance=\allowbreak{}0.0001} \\
\addlinespace[3pt]
\textsc{FedRep} & 0.001953125 & \textemdash{} & \texttt{momentum=\allowbreak{}0.0}; \texttt{n\_\allowbreak{}head\_\allowbreak{}epochs=\allowbreak{}1}; \texttt{n\_\allowbreak{}local\_\allowbreak{}epochs=\allowbreak{}1};\newline \texttt{weight\_\allowbreak{}decay=\allowbreak{}0.0} \\
\addlinespace[3pt]
\textsc{pFedHN} & 0.015625 & \textemdash{} & \texttt{embed\_\allowbreak{}dim=\allowbreak{}20}; \texttt{hyper\_\allowbreak{}hid=\allowbreak{}100}; \texttt{hyper\_\allowbreak{}weight\_\allowbreak{}decay=\allowbreak{}0.0};\newline \texttt{hypernetwork\_\allowbreak{}gradient\_\allowbreak{}clip\_\allowbreak{}norm=\allowbreak{}50}; \texttt{local\_\allowbreak{}gradient\_\allowbreak{}clip\_\allowbreak{}norm=\allowbreak{}50}; \texttt{momentum=\allowbreak{}0.0};\newline \texttt{n\_\allowbreak{}hidden\_\allowbreak{}layer=\allowbreak{}3}; \texttt{n\_\allowbreak{}local\_\allowbreak{}epochs=\allowbreak{}1}; \texttt{weight\_\allowbreak{}decay=\allowbreak{}0.0} \\
\bottomrule
\end{tabularx}
\end{table}

\begin{table}[p]
\centering
\caption{Method hyperparameters for CIFAR10 $\alpha=0.5$.}
\label{tab:hyperparameters-cifar10-alpha-05}
\small
\setlength{\tabcolsep}{4pt}
\renewcommand{\arraystretch}{1.15}
\begin{tabularx}{\textwidth}{@{}lrl>{\raggedright\arraybackslash}X@{}}
\toprule
\textbf{Method} & \shortstack[r]{\textbf{Learning}\\\textbf{rate}} & \shortstack[l]{\textbf{Tuned}\\\textbf{hyperparameter}} & \textbf{Other hyperparameters} \\
\midrule
\textsc{FedAvg} & 0.015625 & \textemdash{} & \texttt{n\_\allowbreak{}local\_\allowbreak{}epochs=\allowbreak{}1} \\
\addlinespace[3pt]
\textsc{pFedMe} & 0.015625 & \texttt{moreau\_\allowbreak{}param=\allowbreak{}10.0} & \texttt{beta=\allowbreak{}1.0}; \texttt{inner\_\allowbreak{}steps=\allowbreak{}5}; \texttt{mu=\allowbreak{}0.001}; \texttt{n\_\allowbreak{}local\_\allowbreak{}epochs=\allowbreak{}1} \\
\addlinespace[3pt]
\textsc{IFCA} & 0.0078125 & \texttt{k\_\allowbreak{}clusters=\allowbreak{}2} & \texttt{aggregation=\allowbreak{}uniform}; \texttt{max\_\allowbreak{}init\_\allowbreak{}attempts=\allowbreak{}100}; \texttt{min\_\allowbreak{}cluster\_\allowbreak{}fraction=\allowbreak{}0.0}; \texttt{n\_\allowbreak{}local\_\allowbreak{}epochs=\allowbreak{}1}; \texttt{variant=\allowbreak{}full} \\
\addlinespace[3pt]
\textsc{FedSoft} & 0.0078125 & \texttt{reg\_\allowbreak{}weight=\allowbreak{}1.0} & \texttt{k\_\allowbreak{}clusters=\allowbreak{}4}; \texttt{n\_\allowbreak{}local\_\allowbreak{}epochs=\allowbreak{}1}; \texttt{refresh\_\allowbreak{}every=\allowbreak{}2}; \texttt{smoothing=\allowbreak{}0.0001} \\
\addlinespace[3pt]
\textsc{pFedGraph} & 0.03125 & \texttt{reg\_\allowbreak{}param=\allowbreak{}0.1} & \texttt{alpha=\allowbreak{}0.8}; \texttt{n\_\allowbreak{}local\_\allowbreak{}epochs=\allowbreak{}1} \\
\addlinespace[3pt]
{\textsc{FedAMP}} & {0.0078125} & {\texttt{lambda\_\allowbreak{}1=\allowbreak{}1.0}} & {\texttt{momentum=\allowbreak{}0.9}; \texttt{n\_\allowbreak{}local\_\allowbreak{}epochs=\allowbreak{}1}; \texttt{self\_\allowbreak{}weight=\allowbreak{}0.3}; \texttt{similarity\_\allowbreak{}scale=\allowbreak{}10.0}; \texttt{weight\_\allowbreak{}decay=\allowbreak{}1e-05}} \\
\addlinespace[3pt]
{\textsc{Ditto}} & {0.125} & {\texttt{lam=\allowbreak{}0.1}} & {\texttt{n\_\allowbreak{}local\_\allowbreak{}epochs=\allowbreak{}1}; \texttt{sampling=\allowbreak{}2}} \\
\addlinespace[3pt]
\textsc{Karula} & 0.0078125 & \texttt{threshold=\allowbreak{}0.001} & \texttt{cost\_\allowbreak{}fn=\allowbreak{}torch.cdist (Euclidean)}; \texttt{initialize\_\allowbreak{}memory=\allowbreak{}all clients before round 1}; \texttt{n\_\allowbreak{}local\_\allowbreak{}epochs=\allowbreak{}1}; \texttt{n\_\allowbreak{}ref=\allowbreak{}1000}; \texttt{projection\_\allowbreak{}max\_\allowbreak{}iter=\allowbreak{}5000}; \texttt{projection\_\allowbreak{}tolerance=\allowbreak{}0.0001} \\
\addlinespace[3pt]
\textsc{FedRep} & 0.0009765625 & \textemdash{} & \texttt{momentum=\allowbreak{}0.0}; \texttt{n\_\allowbreak{}head\_\allowbreak{}epochs=\allowbreak{}1}; \texttt{n\_\allowbreak{}local\_\allowbreak{}epochs=\allowbreak{}1};\newline \texttt{weight\_\allowbreak{}decay=\allowbreak{}0.0} \\
\addlinespace[3pt]
\textsc{pFedHN} & 0.0078125 & \textemdash{} & \texttt{embed\_\allowbreak{}dim=\allowbreak{}20}; \texttt{hyper\_\allowbreak{}hid=\allowbreak{}100}; \texttt{hyper\_\allowbreak{}weight\_\allowbreak{}decay=\allowbreak{}0.0};\newline \texttt{hypernetwork\_\allowbreak{}gradient\_\allowbreak{}clip\_\allowbreak{}norm=\allowbreak{}50}; \texttt{local\_\allowbreak{}gradient\_\allowbreak{}clip\_\allowbreak{}norm=\allowbreak{}50}; \texttt{momentum=\allowbreak{}0.0};\newline \texttt{n\_\allowbreak{}hidden\_\allowbreak{}layer=\allowbreak{}3}; \texttt{n\_\allowbreak{}local\_\allowbreak{}epochs=\allowbreak{}1}; \texttt{weight\_\allowbreak{}decay=\allowbreak{}0.0} \\
\bottomrule
\end{tabularx}
\end{table}

\begin{table}[p]
\centering
\caption{Method hyperparameters for CIFAR100 $\alpha=0.1$.}
\label{tab:hyperparameters-cifar100-alpha-01}
\small
\setlength{\tabcolsep}{4pt}
\renewcommand{\arraystretch}{1.15}
\begin{tabularx}{\textwidth}{@{}lrl>{\raggedright\arraybackslash}X@{}}
\toprule
\textbf{Method} & \shortstack[r]{\textbf{Learning}\\\textbf{rate}} & \shortstack[l]{\textbf{Tuned}\\\textbf{hyperparameter}} & \textbf{Other hyperparameters} \\
\midrule
\textsc{FedAvg} & 0.015625 & \textemdash{} & \texttt{n\_\allowbreak{}local\_\allowbreak{}epochs=\allowbreak{}1} \\
\addlinespace[3pt]
\textsc{pFedMe} & 0.125 & \texttt{moreau\_\allowbreak{}param=\allowbreak{}1.0} & \texttt{beta=\allowbreak{}1.0}; \texttt{inner\_\allowbreak{}steps=\allowbreak{}5}; \texttt{mu=\allowbreak{}0.001}; \texttt{n\_\allowbreak{}local\_\allowbreak{}epochs=\allowbreak{}1} \\
\addlinespace[3pt]
\textsc{IFCA} & 0.03125 & \texttt{k\_\allowbreak{}clusters=\allowbreak{}8} & \texttt{aggregation=\allowbreak{}uniform}; \texttt{max\_\allowbreak{}init\_\allowbreak{}attempts=\allowbreak{}100}; \texttt{min\_\allowbreak{}cluster\_\allowbreak{}fraction=\allowbreak{}0.0}; \texttt{n\_\allowbreak{}local\_\allowbreak{}epochs=\allowbreak{}1}; \texttt{variant=\allowbreak{}full} \\
\addlinespace[3pt]
\textsc{FedSoft} & 0.03125 & \texttt{reg\_\allowbreak{}weight=\allowbreak{}0.01} & \texttt{k\_\allowbreak{}clusters=\allowbreak{}4}; \texttt{n\_\allowbreak{}local\_\allowbreak{}epochs=\allowbreak{}1}; \texttt{refresh\_\allowbreak{}every=\allowbreak{}2}; \texttt{smoothing=\allowbreak{}0.0001} \\
\addlinespace[3pt]
\textsc{pFedGraph} & 0.03125 & \texttt{reg\_\allowbreak{}param=\allowbreak{}0.1} & \texttt{alpha=\allowbreak{}0.8}; \texttt{n\_\allowbreak{}local\_\allowbreak{}epochs=\allowbreak{}1} \\
\addlinespace[3pt]
{\textsc{FedAMP}} & {0.00390625} & {\texttt{lambda\_\allowbreak{}1=\allowbreak{}0.1}} & {\texttt{momentum=\allowbreak{}0.9}; \texttt{n\_\allowbreak{}local\_\allowbreak{}epochs=\allowbreak{}1}; \texttt{self\_\allowbreak{}weight=\allowbreak{}0.3}; \texttt{similarity\_\allowbreak{}scale=\allowbreak{}10.0}; \texttt{weight\_\allowbreak{}decay=\allowbreak{}1e-05}} \\
\addlinespace[3pt]
{\textsc{Ditto}} & {0.125} & {\texttt{lam=\allowbreak{}0.1}} & {\texttt{n\_\allowbreak{}local\_\allowbreak{}epochs=\allowbreak{}1}; \texttt{sampling=\allowbreak{}2}} \\
\addlinespace[3pt]
\textsc{Karula} & 0.0078125 & \texttt{threshold=\allowbreak{}0.1} & \texttt{cost\_\allowbreak{}fn=\allowbreak{}torch.cdist (Euclidean)}; \texttt{initialize\_\allowbreak{}memory=\allowbreak{}all clients before round 1}; \texttt{n\_\allowbreak{}local\_\allowbreak{}epochs=\allowbreak{}1}; \texttt{n\_\allowbreak{}ref=\allowbreak{}1000}; \texttt{projection\_\allowbreak{}max\_\allowbreak{}iter=\allowbreak{}5000}; \texttt{projection\_\allowbreak{}tolerance=\allowbreak{}0.001} \\
\addlinespace[3pt]
\textsc{FedRep} & 0.125 & \textemdash{} & \texttt{momentum=\allowbreak{}0.0}; \texttt{n\_\allowbreak{}head\_\allowbreak{}epochs=\allowbreak{}1}; \texttt{n\_\allowbreak{}local\_\allowbreak{}epochs=\allowbreak{}1};\newline \texttt{weight\_\allowbreak{}decay=\allowbreak{}0.0} \\
\addlinespace[3pt]
\textsc{pFedHN} & 0.125 & \textemdash{} & \texttt{embed\_\allowbreak{}dim=\allowbreak{}20}; \texttt{hyper\_\allowbreak{}hid=\allowbreak{}100}; \texttt{hyper\_\allowbreak{}weight\_\allowbreak{}decay=\allowbreak{}0.0};\newline \texttt{hypernetwork\_\allowbreak{}gradient\_\allowbreak{}clip\_\allowbreak{}norm=\allowbreak{}50}; \texttt{local\_\allowbreak{}gradient\_\allowbreak{}clip\_\allowbreak{}norm=\allowbreak{}50}; \texttt{momentum=\allowbreak{}0.0};\newline \texttt{n\_\allowbreak{}hidden\_\allowbreak{}layer=\allowbreak{}3}; \texttt{n\_\allowbreak{}local\_\allowbreak{}epochs=\allowbreak{}1}; \texttt{weight\_\allowbreak{}decay=\allowbreak{}0.0} \\
\bottomrule
\end{tabularx}
\end{table}

\begin{table}[H]
\centering
\caption{Method hyperparameters for CIFAR100 $\alpha=0.5$.}
\label{tab:hyperparameters-cifar100-alpha-05}
\small
\setlength{\tabcolsep}{4pt}
\renewcommand{\arraystretch}{1.15}
\begin{tabularx}{\textwidth}{@{}lrl>{\raggedright\arraybackslash}X@{}}
\toprule
\textbf{Method} & \shortstack[r]{\textbf{Learning}\\\textbf{rate}} & \shortstack[l]{\textbf{Tuned}\\\textbf{hyperparameter}} & \textbf{Other hyperparameters} \\
\midrule
\textsc{FedAvg} & 0.0078125 & \textemdash{} & \texttt{n\_\allowbreak{}local\_\allowbreak{}epochs=\allowbreak{}1} \\
\addlinespace[3pt]
\textsc{pFedMe} & 0.125 & \texttt{moreau\_\allowbreak{}param=\allowbreak{}10.0} & \texttt{beta=\allowbreak{}1.0}; \texttt{inner\_\allowbreak{}steps=\allowbreak{}5}; \texttt{mu=\allowbreak{}0.001}; \texttt{n\_\allowbreak{}local\_\allowbreak{}epochs=\allowbreak{}1} \\
\addlinespace[3pt]
\textsc{IFCA} & 0.125 & \texttt{k\_\allowbreak{}clusters=\allowbreak{}2} & \texttt{aggregation=\allowbreak{}uniform}; \texttt{max\_\allowbreak{}init\_\allowbreak{}attempts=\allowbreak{}100}; \texttt{min\_\allowbreak{}cluster\_\allowbreak{}fraction=\allowbreak{}0.0}; \texttt{n\_\allowbreak{}local\_\allowbreak{}epochs=\allowbreak{}1}; \texttt{variant=\allowbreak{}full} \\
\addlinespace[3pt]
\textsc{FedSoft} & 0.125 & \texttt{reg\_\allowbreak{}weight=\allowbreak{}1.0} & \texttt{k\_\allowbreak{}clusters=\allowbreak{}4}; \texttt{n\_\allowbreak{}local\_\allowbreak{}epochs=\allowbreak{}1}; \texttt{refresh\_\allowbreak{}every=\allowbreak{}2}; \texttt{smoothing=\allowbreak{}0.0001} \\
\addlinespace[3pt]
\textsc{pFedGraph} & 0.125 & \texttt{reg\_\allowbreak{}param=\allowbreak{}1.0} & \texttt{alpha=\allowbreak{}0.8}; \texttt{n\_\allowbreak{}local\_\allowbreak{}epochs=\allowbreak{}1} \\
\addlinespace[3pt]
{\textsc{FedAMP}} & {0.03125} & {\texttt{lambda\_\allowbreak{}1=\allowbreak{}1.0}} & {\texttt{momentum=\allowbreak{}0.9}; \texttt{n\_\allowbreak{}local\_\allowbreak{}epochs=\allowbreak{}1}; \texttt{self\_\allowbreak{}weight=\allowbreak{}0.3}; \texttt{similarity\_\allowbreak{}scale=\allowbreak{}10.0}; \texttt{weight\_\allowbreak{}decay=\allowbreak{}1e-05}} \\
\addlinespace[3pt]
{\textsc{Ditto}} & {0.125} & {\texttt{lam=\allowbreak{}0.1}} & {\texttt{n\_\allowbreak{}local\_\allowbreak{}epochs=\allowbreak{}1}; \texttt{sampling=\allowbreak{}2}} \\
\addlinespace[3pt]
\textsc{Karula} & 0.125 & \texttt{threshold=\allowbreak{}1e-05} & \texttt{cost\_\allowbreak{}fn=\allowbreak{}torch.cdist (Euclidean)}; \texttt{initialize\_\allowbreak{}memory=\allowbreak{}all clients before round 1}; \texttt{n\_\allowbreak{}local\_\allowbreak{}epochs=\allowbreak{}1}; \texttt{n\_\allowbreak{}ref=\allowbreak{}1000}; \texttt{projection\_\allowbreak{}max\_\allowbreak{}iter=\allowbreak{}5000}; \texttt{projection\_\allowbreak{}tolerance=\allowbreak{}0.001} \\
\addlinespace[3pt]
\textsc{FedRep} & 0.125 & \textemdash{} & \texttt{momentum=\allowbreak{}0.0}; \texttt{n\_\allowbreak{}head\_\allowbreak{}epochs=\allowbreak{}1}; \texttt{n\_\allowbreak{}local\_\allowbreak{}epochs=\allowbreak{}1};\newline \texttt{weight\_\allowbreak{}decay=\allowbreak{}0.0} \\
\addlinespace[3pt]
\textsc{pFedHN} & 0.125 & \textemdash{} & \texttt{embed\_\allowbreak{}dim=\allowbreak{}20}; \texttt{hyper\_\allowbreak{}hid=\allowbreak{}100}; \texttt{hyper\_\allowbreak{}weight\_\allowbreak{}decay=\allowbreak{}0.0};\newline \texttt{hypernetwork\_\allowbreak{}gradient\_\allowbreak{}clip\_\allowbreak{}norm=\allowbreak{}50}; \texttt{local\_\allowbreak{}gradient\_\allowbreak{}clip\_\allowbreak{}norm=\allowbreak{}50}; \texttt{momentum=\allowbreak{}0.0};\newline \texttt{n\_\allowbreak{}hidden\_\allowbreak{}layer=\allowbreak{}3}; \texttt{n\_\allowbreak{}local\_\allowbreak{}epochs=\allowbreak{}1}; \texttt{weight\_\allowbreak{}decay=\allowbreak{}0.0} \\
\bottomrule
\end{tabularx}
\end{table}

\begin{table}[H]
\centering
\caption{Architecture of the CNN used for FEMNIST classification.}
\label{tab:cnn-femnist}
\begin{tabular}{lll}
\hline
\textbf{Layer} & \textbf{Operation} & \textbf{Output Shape} \\
\hline
Input & Grayscale image & $1 \times 28 \times 28$ \\
Conv1 & Conv2d(1, 32, kernel=3, pad=1) + ReLU & $32 \times 28 \times 28$ \\
Pool1 & MaxPool2d(2,2) & $32 \times 14 \times 14$ \\
Conv2 & Conv2d(32, 64, kernel=3, pad=1) + ReLU & $64 \times 14 \times 14$ \\
Pool2 & MaxPool2d(2,2) & $64 \times 7 \times 7$ \\
Flatten & Vectorization & $64 \cdot 7 \cdot 7$ \\
FC1 & Linear($64\cdot7\cdot7$, 128) + ReLU & $128$ \\
FC2 & Linear(128, 62) & 62 \\
\hline
\end{tabular}
\end{table}

\begin{table}[H]
\centering
\caption{Architecture of the CNN used for CIFAR-$K$ classification.}
\label{tab:cnn-cifar}
\begin{tabular}{lll}
\hline
\textbf{Layer} & \textbf{Operation} & \textbf{Output Shape} \\
\hline
Input & RGB image & $3 \times 32 \times 32$ \\
Conv1 & Conv2d(3, 32, kernel=5, padding=2) + BatchNorm + ReLU & $32 \times 32 \times 32$ \\
Pool1 & MaxPool2d(2,2) & $32 \times 16 \times 16$ \\
Conv2 & Conv2d(32, 64, kernel=5, padding=2) + BatchNorm + ReLU & $64 \times 16 \times 16$ \\
Pool2 & MaxPool2d(2,2) & $64 \times 8 \times 8$ \\
Flatten & Vectorization & $64 \cdot 8 \cdot 8$ \\
FC1 & Linear($64\cdot8\cdot8$, 128) + ReLU & $128$ \\
FC2 & Linear(128, $K$) & $K$ \\
\hline
\end{tabular}
\end{table}

\newpage

\section{Proof of theoretical results}
\label{sec:proof}

\subsection{Proof of Theorem~\ref{thm:dissim-bound}}
\begin{proof}
    Define $\mathcal{L}_1$ as the space of real-valued $1$-Lipschitz functions
    on the data space $Z$. Then, under Assumption~\ref{asmp:lipschitz} we can
    apply Kantorovich--Rubinstein Duality \citep[Theorem~5.10]{Villani08}
    to conclude that for all $i$ and $j$,
    \begin{equation}\label{eq:krd}
        \begin{aligned}
            f(\theta_i^\star; \mu_i) - f(\theta_i^\star; \mu_j) &= \E_{(x,y) \sim \mu_i} [\ell(\phi_{\theta_i^\star}(x), y)] - \E_{(x,y) \sim \mu_j} [\ell(\phi_{\theta_i^\star}(x), y)] \\
                                                                &\leq L_\ell \sup_{g \in\mathcal{L}_1}\left\{ \E_{z\sim \mu_i}[g(z)] -\E_{z \sim\mu_j} [g(z)]\right\} \\
                                                                &= L_\ell \mathcal{W}_1(\mu_i, \mu_j).
        \end{aligned}
    \end{equation}
    Under Assumption~\ref{asmp:quadratic-growth}, we have that
    \[
        \begin{aligned}
            f(\theta_j^\star; \mu_i) - f(\theta_i^\star; \mu_i) &\geq \gamma \|W_b^{(i)\star} - W_b^{(j)\star}\|_F^2 \\
            f(\theta_i^\star; \mu_j) - f(\theta_j^\star; \mu_j) &\geq \gamma \|W_b^{(i)\star} - W_b^{(j)\star}\|_F^2. \\
        \end{aligned}
    \]
    Summing these inequalities and using \eqref{eq:krd}, we have that 
    \[
        \begin{aligned}
            2\gamma \|W_b^{(i)\star} - W_b^{(j)\star}\|_F^2 & \leq f(\theta_j^\star; \mu_i) - f(\theta_i^\star; \mu_i) + f(\theta_i^\star; \mu_j) - f(\theta_j^\star; \mu_j)\\ 
                      & \leq 2 L_\ell \mathcal{W}_1(\mu_i, \mu_j).
        \end{aligned}
    \]
    Dividing by $2\gamma > 0$, we obtain the stated result.
\end{proof}

\subsection{Technical lemmas for Theorem~\ref{thm:nonconvex}}

\begin{lem}\label{lem:var-red}
    Under Assumptions~\ref{asmp:smooth} and \ref{asmp:bounded-variance}, the
    iterates $\theta^k$, $\alpha^k$ and $\nu^k$ of Algorithm~\ref{algo:karula}
    satisfy
    \[
        \E \|\nu^k - \nabla F(\theta^k)\|^2
        \leq \frac{2L^2}{ns}\E \|\theta^k - \alpha^k\|^2
        + \frac{4(2n-s)}{ns}\sigma^2.
    \]
\end{lem}

\begin{proof}
    Let $\mathcal{F}_k$ contain the history before sampling $S_k$, including
    $\theta^k$, $\alpha^k$, and $g^k$. Write
    $I_i^k = \mathbf{1}_{\{i\in S_k\}}$ and
    $f_i(x) = f(x;\mathcal{D}_i)$. Conditional on $\mathcal{F}_k$,
    $\Pr(I_i^k=1\mid\mathcal{F}_k)=s/n$. Define the ideal correction and
    direction by
    \[
        \Delta_i^k = \frac{I_i^k}{s}
        \left(\nabla f_i(\theta_i^k)-\nabla f_i(\alpha_i^k)\right),
        \qquad
        \widetilde\nu^k = \nabla F(\alpha^k)+\Delta^k.
    \]
    Since $F(\theta)=\frac{1}{n}\sum_{i=1}^n f_i(\theta_i)$,
    \[
        \E[\Delta^k\mid\mathcal{F}_k]
        =\nabla F(\theta^k)-\nabla F(\alpha^k).
    \]
    Thus $\widetilde\nu^k$ is conditionally unbiased for $\nabla F(\theta^k)$.
    The conditional variance identity and $L$-smoothness of the local
    objectives give
    \[
        \begin{aligned}
            \E\!\left[\|\widetilde\nu^k-\nabla F(\theta^k)\|^2
                \mid\mathcal{F}_k\right]
            &=\E\!\left[
                \|\Delta^k-\E[\Delta^k\mid\mathcal{F}_k]\|^2
                \mid\mathcal{F}_k\right] \\
            &\leq\E[\|\Delta^k\|^2\mid\mathcal{F}_k] \\
            &=\frac{1}{ns}\sum_{i=1}^n
                \|\nabla f_i(\theta_i^k)-\nabla f_i(\alpha_i^k)\|^2 \\
            &\leq\frac{L^2}{ns}\|\theta^k-\alpha^k\|^2.
        \end{aligned}
    \]

    To control the oracle error, define the stored error
    $r_i^k = g_i^k-\nabla f_i(\alpha_i^k)$ and, for
    $i\in S_k$, the fresh error
    $\xi_i^k=G_i(\theta_i^k)-\nabla f_i(\theta_i^k)$. Set $\xi_i^k=0$ for
    $i\notin S_k$. The simultaneous refresh of
    $g_i^{k+1}$ and $\alpha_i^{k+1}$ implies
    \[
        r_i^{k+1} =
        \begin{cases}
            \xi_i^k,&i\in S_k,\\
            r_i^k,&i\notin S_k.
        \end{cases}
    \]
    Assumption~\ref{asmp:bounded-variance} and uniform sampling yield
    \[
        \E\|r_i^{k+1}\|^2
        \leq \left(1-\frac{s}{n}\right)\E\|r_i^k\|^2
        +\frac{s}{n}\sigma^2.
    \]
    The initialization satisfies $\E\|r_i^0\|^2\leq\sigma^2$, so induction
    gives $\E\|r_i^k\|^2\leq\sigma^2$ for every $i$ and $k$.

    The raw gradient table and the direction $\nu^k=g^k/n+(g^{k+1}-g^k)/s$ in
    Algorithm~\ref{algo:karula} give
    \[
        \nu_i^k-\widetilde\nu_i^k
        =\left(\frac1n-\frac{I_i^k}{s}\right)r_i^k
        +\frac{I_i^k}{s}\xi_i^k.
    \]
    Since $r_i^k$ is $\mathcal{F}_k$-measurable and
    \[
        \E\!\left[
        \left(\frac1n-\frac{I_i^k}{s}\right)^2\mid\mathcal{F}_k\right]
        =\frac{n-s}{n^2s},
    \]
    Young's inequality and the conditional oracle bound imply
    \[
        \begin{aligned}
            \E\|\nu^k-\widetilde\nu^k\|^2
            &\leq\frac{2(n-s)}{n^2s}
                \sum_{i=1}^n\E\|r_i^k\|^2
                +\frac{2}{s^2}\E\sum_{i\in S_k}\|\xi_i^k\|^2 \\
            &\leq\frac{2(n-s)}{ns}\sigma^2
                +\frac{2}{s}\sigma^2
                =\frac{2(2n-s)}{ns}\sigma^2.
        \end{aligned}
    \]
    Combining the two bounds, we obtain
    \[
        \begin{aligned}
            \E\|\nu^k-\nabla F(\theta^k)\|^2
            &\leq 2\E\|\widetilde\nu^k-\nabla F(\theta^k)\|^2
                +2\E\|\nu^k-\widetilde\nu^k\|^2 \\
            &\leq\frac{2L^2}{ns}\E\|\theta^k-\alpha^k\|^2
                +\frac{4(2n-s)}{ns}\sigma^2.
        \end{aligned}
    \]
\end{proof}

\begin{lem}\label{lem:prox}
    Under Assumptions~\ref{asmp:smooth} and \ref{asmp:proj}, for $u =
    \hat{\Pi}_\mathcal{K} (\theta - \eta \nu)$ and $\bar u = \Pi_\mathcal{K}
    (\theta - \eta \nu)$, with any $\eta>0$, $\theta\in\mathcal{K}$, and
    $\nu$ in the ambient parameter space,
    the inequality 
    \begin{equation}\label{eq:prox-ineq}
        \begin{aligned}
            F(u) & \leq F(v) + \braket{u - v, \nabla F(\theta) - \nu} + \left( \frac{L}{2n} - \frac{1}{2 \eta} \right) \|u - \theta\|^2 \\
                 &+ \left( \frac{L}{2n} + \frac{1}{2 \eta} \right) \|v - \theta\|^2 - \frac{1}{2 \eta} \|\bar u - v\|^2 + \delta
        \end{aligned}
    \end{equation}
    holds for all $v \in \mathcal{K}$. Moreover,
    \begin{equation}\label{eq:diff-bound}
        \frac{1}{2 \eta} \|u - \bar u\|^2 \leq \delta.
    \end{equation}
\end{lem}

\begin{proof}
    By $L/n$-smoothness of $F$, the two inequalities
    \[
        F(u) \leq F(\theta) + \braket{\nabla F(\theta), u - \theta} + \frac{L}{2n}\|u - \theta\|^2 
    \]
    and 
    \[
        F(\theta) \leq F(v) + \braket{\nabla F(\theta), \theta - v} + \frac{L}{2n}\|v - \theta\|^2 
    \]
    hold for any $v \in \mathcal{K}$. Adding the two inequalities, we have that
    \begin{equation}\label{eq:smooth-ineq}
        F(u) \leq F(v) + \braket{\nabla F(\theta), u - v} + \frac{L}{2n} (\|u - \theta\|^2 + \|v - \theta\|^2).
    \end{equation}
    Since $1/(2\eta)\|v-(\theta-\eta\nu)\|^2$ is $1/\eta$-strongly
    convex in $v$ and $\bar u$ minimizes it over the convex set
    $\mathcal{K}$, the following inequality holds for all $v\in\mathcal{K}$:
    \[
        \frac{1}{2 \eta} \|\bar u - (\theta - \eta \nu)\|^2 + \frac{1}{2 \eta} \|\bar u - v\|^2 \leq \frac{1}{2 \eta} \|v - (\theta - \eta \nu)\|^2.
    \]
    We also have that
    \begin{equation}\label{eq:inexact-prox}
        \frac{1}{2 \eta} \|u - (\theta - \eta \nu)\|^2 - \delta \leq \frac{1}{2 \eta} \|\bar u - (\theta - \eta \nu)\|^2,
    \end{equation}
    so,
    \[
        \frac{1}{2 \eta} \|u - (\theta - \eta \nu)\|^2 + \frac{1}{2 \eta} \|\bar u - v\|^2 \leq \frac{1}{2 \eta} \|v - (\theta - \eta \nu)\|^2 + \delta.
    \]
    Expanding squares and moving all terms to the right-hand side, we obtain
    the inequality 
    \[
        0 \leq -\braket{u - v, \nu} + \frac{1}{2 \eta} (\|v - \theta\|^2 - \|u - \theta\|^2 - \|\bar u - v\|^2) + \delta.
    \]
    Adding this inequality to inequality \eqref{eq:smooth-ineq}, we arrive at
    the first result \eqref{eq:prox-ineq}.

    For the second result \eqref{eq:diff-bound}, note that due to strong
    convexity,
    \[
        \frac{1}{2 \eta} \|\bar u - (\theta - \eta \nu)\|^2 + \frac{1}{2 \eta} \|\bar u - u\|^2 \leq \frac{1}{2 \eta} \|u - (\theta - \eta \nu)\|^2.
    \]
    Using this coupled with \eqref{eq:inexact-prox}, we obtain the stated
    inequality.
\end{proof}

    % \[
    %     F(u) \leq F(v) + \braket{u - v, \nabla F(\theta) - \nu} + \left( \frac{L}{2} - \frac{1}{2 \eta} \right) \|u - \theta\|^2 + \left( \frac{L}{2} + \frac{1}{2 \eta} \right) \|v - \theta\|^2 - \frac{1}{2 \eta} \|\bar u - v\|^2 + \delta.
    % \]

    % \[
    %     r(u) = I_\mathcal{K}(u) + \frac{1}{2 \eta} \|u - (\theta - \eta \nu)\|^2, \quad I_\mathcal{K}(u) = \begin{cases}
    %         0      & \text{if } u \in \mathcal{K}, \\
    %         \infty & \text{otherwise}.
    %     \end{cases}
    % \]

\subsection{Proof of Theorem~\ref{thm:nonconvex}}

\begin{proof}%[Proof of Theorem~\ref{thm:nonconvex}]
    Let $\vartheta^{k+1} = \Pi_\mathcal{K} (\theta^k - \eta \nabla
    F(\theta^k))$ be the full projected gradient update and $\bar \theta^{k+1}
    = \Pi_\mathcal{K} (\theta^k - \eta \nu^k)$ be the exact projected update.
    Using \eqref{eq:prox-ineq} of Lemma~\ref{lem:prox}, we have that with $u =
    \theta^{k+1}$, $v = \vartheta^{k+1}$ and $\nu = \nu^k$,
    \begin{equation}\label{eq:bound1}
        \begin{aligned}
            \E [F(\theta^{k+1})] \leq \E \bigg[ & F(\vartheta^{k+1}) + \braket{\theta^{k+1} - \vartheta^{k+1}, \nabla F(\theta^k) - \nu^k} + \left(\frac{L}{2n} - \frac{1}{2\eta}\right)\|\theta^{k+1} - \theta^k\|^2 \\
                                                & + \left( \frac{L}{2n} + \frac{1}{2 \eta}  \right) \|\vartheta^{k+1} - \theta^k\|^2 - \frac{1}{2 \eta} \|\bar \theta^{k+1} - \vartheta^{k+1}\|^2 + \delta \bigg],
        \end{aligned}
    \end{equation}
    and, applying the exact-projection case of the lemma with $\delta=0$,
    $u = \vartheta^{k+1}$, $v = \theta^k$, and $\nu = \nabla F(\theta^k)$,
    \begin{equation}\label{eq:bound2}
        \E [F(\vartheta^{k+1})] \leq \E \left[ F(\theta^k) + \left( \frac{L}{2n} - \frac{1}{2 \eta} \right) \|\vartheta^{k+1} - \theta^k\|^2 - \frac{1}{2 \eta} \|\vartheta^{k+1} - \theta^k\|^2 \right]. 
    \end{equation}
    To bound the inner product, we apply the Cauchy-Schwarz inequality and
    Young's inequality as follows:
    \[
        \begin{aligned}
            \E [ \braket{\theta^{k+1} - \vartheta^{k+1},& \nabla F(\theta^k) - \nu^k} ]  \leq \E \left[ \|\theta^{k+1} - \vartheta^{k+1}\| \|\nabla F(\theta^k) - \nu^k \| \right] \\
                                                                                     & \leq \E \left[ \frac{1}{4 \eta} \| \theta^{k+1} - \vartheta^{k+1}\|^2 + \eta \|\nabla F(\theta^k) - \nu^k\|^2 \right] \\
                                                                                     & \leq \E \left[ \frac{1}{2 \eta} (\|\bar \theta^{k+1} - \vartheta^{k+1}\|^2 + \|\bar \theta^{k+1} - \theta^{k+1}\|^2) + \eta \|\nabla F(\theta^k) - \nu^k\|^2 \right] \\
                                                                                     & \leq \E \left[ \frac{1}{2 \eta} \|\bar \theta^{k+1} - \vartheta^{k+1}\|^2 + \eta \|\nabla F(\theta^k) - \nu^k\|^2 + \delta \right],
        \end{aligned}
    \]
    where in the last inequality, we applied Equation~\eqref{eq:diff-bound}
    from Lemma~\ref{lem:prox}. Applying Lemma~\ref{lem:var-red}, we obtain the
    bound 
    \[
        \E [ \braket{\theta^{k+1} - \vartheta^{k+1}, \nabla F(\theta^k) - \nu^k} ] \leq \frac{1}{2 \eta} \E \|\bar \theta^{k+1} - \vartheta^{k+1}\|^2 + \frac{2\eta L^2}{ns} \E \|\theta^k - \alpha^k\|^2 + \frac{4\eta(2n-s)}{ns} \sigma^2 + \delta.
    \]
    Adding the inequalities \eqref{eq:bound1} and \eqref{eq:bound2} and using
    the bound we just derived, we have that
    \begin{equation}\label{eq:master-ineq}
        \begin{aligned}
            \E [F(\theta^{k+1})] \leq \E \bigg[ & F(\theta^k) + \frac{1}{2 \eta} \|\bar\theta^{k+1} - \vartheta^{k+1}\|^2 + \frac{2\eta L^2}{ns} \|\theta^k - \alpha^k\|^2 + \frac{4\eta(2n-s)}{ns} \sigma^2 + 2\delta \\
                                                & + \left( \frac{L}{n} - \frac{1}{2 \eta} \right) \|\vartheta^{k+1} - \theta^k\|^2 + \left( \frac{L}{2n} - \frac{1}{2 \eta} \right) \|\theta^{k+1} - \theta^k\|^2 - \frac{1}{2 \eta} \|\bar\theta^{k+1} - \vartheta^{k+1}\|^2 \bigg] \\
                                 =    \E \bigg[ & F(\theta^k) + \frac{2\eta L^2}{ns} \|\theta^k - \alpha^k\|^2 + \frac{4\eta(2n-s)}{ns} \sigma^2 + 2\delta \\
                                                & + \eta^2 \left( \frac{L}{n} - \frac{1}{2 \eta} \right) \|\mathbf{G}_\eta (\theta^k)\|^2 + \left( \frac{L}{2n} - \frac{1}{2 \eta} \right) \|\theta^{k+1} - \theta^k\|^2 \bigg].
        \end{aligned}
    \end{equation}
    where we used that $\mathbf{G}_\eta(\theta^k) = (\theta^k -
    \vartheta^{k+1}) / \eta$. Now, we introduce the Lyapunov function
    \[
        V^k = \E \left[ F(\theta^k) + c_k \|\theta^k - \alpha^k\|^2 \right].
    \]
    where $c_k \geq 0$ will be set in the sequel. Let $\mathcal{F}_k$ denote
    the history before sampling the clients at iteration $k$, so that
    $\theta^k$ and $\alpha^k$ are $\mathcal{F}_k$-measurable. Conditional on
    this history, $S_k$ is a uniform subset of size $s$. Define
    $d^k = \theta^{k+1} - \theta^k$, $e^k = \theta^k - \alpha^k$, and let
    $Q_k$ retain the client blocks outside $S_k$ and set the other blocks to
    zero. The memory update gives the pointwise identity
    \[
        \theta^{k+1} - \alpha^{k+1} = d^k + Q_k e^k.
    \]
    Although $d^k$ may depend on $S_k$, Young's inequality applies pointwise:
    for every $\beta > 0$,
    \[
        \|\theta^{k+1} - \alpha^{k+1}\|^2
        \leq (1 + \beta^{-1})\|d^k\|^2
        + (1 + \beta)\|Q_k e^k\|^2.
    \]
    Since $e^k$ is fixed conditional on $\mathcal{F}_k$, uniform sampling gives
    \[
        \E\left[\|Q_k e^k\|^2 \mid \mathcal{F}_k\right]
        = \sum_{i=1}^n \Pr(i \notin S_k \mid \mathcal{F}_k)\|e_i^k\|^2
        = \left(1 - \frac{s}{n}\right)\|e^k\|^2.
    \]
    Taking expectations and choosing $\beta = s/n$, we obtain
    \[
        \E\|\theta^{k+1} - \alpha^{k+1}\|^2
        \leq \left(1 + \frac{n}{s}\right)\E\|\theta^{k+1} - \theta^k\|^2
        + \left(1 - \frac{s^2}{n^2}\right)\E\|\theta^k - \alpha^k\|^2.
    \]
    This bound does not require independence between $\theta^{k+1}$ and
    $S_k$. Combining it with inequality \eqref{eq:master-ineq}, it follows that
    \[
        \begin{aligned}
            V^{k+1} \leq \E \bigg[ & F(\theta^{k+1}) + c_{k+1} \left(1 + \frac{n}{s}\right) \|\theta^{k+1} - \theta^k\|^2 + c_{k+1} \frac{n^2-s^2}{n^2} \|\theta^k - \alpha^k\|^2 \bigg] \\
                    \leq \E \bigg[ & F(\theta^k) + \frac{2\eta L^2}{ns} \|\theta^k - \alpha^k\|^2 + \frac{4\eta(2n-s)}{ns} \sigma^2 + 2\delta \\
                                   & + \eta^2 \left( \frac{L}{n} - \frac{1}{2 \eta} \right) \|\mathbf{G}_\eta (\theta^k)\|^2 + \left( \frac{L}{2n} - \frac{1}{2 \eta} \right) \|\theta^{k+1} - \theta^k\|^2 \\
                                   & + c_{k+1} \left(1 + \frac{n}{s}\right) \|\theta^{k+1} - \theta^k\|^2 + c_{k+1} \frac{n^2-s^2}{n^2} \|\theta^k - \alpha^k\|^2 \bigg] \\
                    =    \E \bigg[ & F(\theta^k) + \eta^2 \left( \frac{L}{n} - \frac{1}{2 \eta} \right) \|\mathbf{G}_\eta (\theta^k)\|^2 + \left( \frac{L}{2n} - \frac{1}{2 \eta} + c_{k+1}\left(1 + \frac{n}{s}\right) \right) \|\theta^{k+1} - \theta^k\|^2 \\
                                   & + \left( \frac{2\eta L^2}{ns} + c_{k+1} \frac{n^2-s^2}{n^2} \right) \|\theta^k - \alpha^k\|^2 + \frac{4\eta(2n-s)}{ns} \sigma^2 + 2\delta \bigg].
        \end{aligned}
    \]
    Now, setting $c_K = 0$ and recursively defining
    \[
        c_k = \frac{2\eta L^2}{ns}
        + c_{k+1} \frac{n^2 - s^2}{n^2}
    \]
    it holds that
    \[
        c_k \leq c_0 = \frac{2\eta L^2}{ns}
        \frac{1-\left(\frac{n^2-s^2}{n^2}\right)^K}
             {1-\frac{n^2-s^2}{n^2}}
        \leq \frac{2\eta nL^2}{s^3}
    \]
    for all $k=0,\dots,K$, hence
    \[
        \frac{L}{2n} - \frac{1}{2\eta}
        + c_{k+1}\left(1+\frac{n}{s}\right)
        \leq \frac{L}{2n} - \frac{1}{2\eta}
        + \frac{2\eta n(n+s)L^2}{s^4}.
    \]
    To see that the right-hand side is non-positive under the stated step-size
    restriction, let $\tau=\eta L/n$. Then $\tau\leq s^2/(4n^2)\leq1/4$, and
    \[
        \begin{aligned}
            2\eta\left(\frac{L}{2n}-\frac{1}{2\eta}
                +\frac{2\eta n(n+s)L^2}{s^4}\right)
            &=\tau-1+\frac{4n^3(n+s)}{s^4}\tau^2 \\
            &\leq\tau-1+\frac{n+s}{4n}
                \leq-\frac14<0.
        \end{aligned}
    \]
    Thus,
    \[
        \begin{aligned}
            V^{k+1} &\leq \E \left[ F(\theta^k) + \eta^2 \left( {\frac{L}{n}} - \frac{1}{2 \eta} \right) \|\mathbf{G}_\eta(\theta^k)\|^2 + c_k \|\theta^k - \alpha^k\|^2 + {\frac{4\eta(2n-s)}{ns}} \sigma^2 + 2\delta \right] \\
                    &= V^k - \frac{\eta}{2}{(1-2\eta L/n)}\E\|\mathbf{G}_\eta(\theta^k)\|^2
                       +{\frac{4\eta(2n-s)}{ns}}\sigma^2+2\delta.
        \end{aligned}
    \]
    Since $c_K=0$, we have $V^K=\E[F(\theta^K)]$. The initialization and
    $\alpha^0=\theta^0$ give $V^0=F(\theta^0)$. Summing over $k=0,\dots,K-1$
    and using $F(\theta^K)\geq\inf_{\theta\in\mathcal{K}}F(\theta)$ yields
    \[
        \begin{aligned}
            \frac{\eta}{2}{(1-2\eta L/n)}\sum_{k=0}^{K-1}
                \E\|\mathbf{G}_\eta(\theta^k)\|^2
            &\leq F(\theta^0)-\E[F(\theta^K)]
                +K\left({\frac{4\eta(2n-s)}{ns}}\sigma^2+2\delta\right) \\
            &\leq \Delta+K\left({\frac{4\eta(2n-s)}{ns}}\sigma^2+2\delta\right).
        \end{aligned}
    \]
    Dividing by $\eta K{(1-2\eta L/n)}/2>0$, we obtain
    \[
        \frac{1}{K}\sum_{k=0}^{K-1}\E\|\mathbf{G}_\eta(\theta^k)\|^2
        \leq {\frac{2}{1-2\eta L/n}}
        \left(\frac{\Delta}{\eta K}+{\frac{4(2n-s)}{ns}}\sigma^2
            +\frac{2\delta}{\eta}\right),
    \]
    which is the stated result.
\end{proof}

\newpage
\section{Computational details of the projection step}
\label{sec:projection}

\paragraph{Primal problem.} The projection $\Pi_\mathcal{K}(\vartheta)$
separates into averaging the shared parameters across clients and projecting
the output-layer parameters onto the pairwise constraints. In this section,
$\theta_i,\vartheta_i\in\reals^{p_L}$ denote only the vectorized output-layer
parameters. We index each unordered pair once, using $1\leq i<j\leq n$, and
write $r_{ij}=\sqrt{tD_{ij}}$, where $D_{ij}=D_{ji}\geq0$ and $t\geq0$. The
output-layer projection is the unique solution to
\begin{equation}
    \begin{array}{ll} 
        \underset{\theta_i\in\mathbf{R}^{p_L}}{\mbox{minimize}}
                     & \frac{1}{2}\sum_{i=1}^{n}\|\theta_i-\vartheta_i\|^2 \\
    \mbox{subject to} & \|\theta_i-\theta_j\|\leq r_{ij},\quad 1\leq i<j\leq n.
    \end{array}
\end{equation}
To apply ADMM to solve this problem, consider the following equivalent
optimization problem with variables $\theta_i,z_{ij}\in\reals^{p_L}$,
\begin{equation}\label{eq:primal}
    \begin{array}{ll} 
    \mbox{minimize}   & \frac{1}{2}\sum_{i=1}^{n}\|\theta_i-\vartheta_i\|^2 \\
    \mbox{subject to} & z_{ij}=\theta_i-\theta_j,\quad \|z_{ij}\|\leq r_{ij},\quad 1\leq i<j\leq n.
    \end{array}
\end{equation}
Define $\Pi_{ij}$ as the Euclidean projection onto the ball
$B_{ij}=\{\nu\in\reals^{p_L}\colon\|\nu\|\leq r_{ij}\}$. With penalty parameter
$\rho>0$, the scaled ADMM updates are
\[
    \begin{aligned}
        \theta^{k+1} &= \argmin_{\theta_i\in\mathbf{R}^{p_L}}
        \left\{\frac{1}{2}\sum_{i=1}^{n}\|\theta_i-\vartheta_i\|^2 + \frac{\rho}{2}\sum_{i<j}
        \|\theta_i-\theta_j-z_{ij}^k+u_{ij}^k\|^2\right\} \\
        z_{ij}^{k+1} &= \Pi_{ij}(\theta_i^{k+1} - \theta_j^{k+1} + u_{ij}^k) \\
        u_{ij}^{k+1} &= u_{ij}^k + \theta_i^{k+1} - \theta_j^{k+1} - z_{ij}^{k+1},
        \qquad i<j,
    \end{aligned}
\]
where $u_{ij}^k$ is a scaled dual iterate. There are
$n(n-1)/2$ vector variables of each type, $z_{ij}$ and $u_{ij}$, and their
updates can be computed in parallel across pairs. The projection in the
$z_{ij}$-update is the clipping operation
\[
    \mathbf{clip}_{ij}(\nu)=\Pi_{ij}(\nu)=
    \begin{cases}
        \nu, & \|\nu\|\leq r_{ij}, \\
        \dfrac{r_{ij}}{\|\nu\|}\nu, & \|\nu\|>r_{ij}.
    \end{cases}
\]
This definition also covers $\nu=0$ and $r_{ij}=0$ without division by zero.
Each clipping operation requires $O(p_L)$ work.

The first-order optimality condition for the $\theta$-update gives the linear
system
\[
    \big((1+\rho n)I_n-\rho\mathbf{1}_n\mathbf{1}_n^{\top}\big)\theta^{k+1}=v^k,
\]
where $\mathbf{1}_n\in\reals^n$ is the all-ones vector, and
\[
    v^k=(v_1^k,\dots,v_n^k),\qquad
    v_i^k=\vartheta_i+\rho\left(
    \sum_{j<i}(u_{ji}^k-z_{ji}^k)-\sum_{j>i}(u_{ij}^k-z_{ij}^k)\right).
\]
The coefficient matrix is positive definite and is a diagonal matrix plus a
rank-one matrix. Applying the Sherman--Morrison formula gives
\[
    \theta^{k+1}=\frac{1}{1+\rho n} \left(I_n + \rho\mathbf{1}_n\mathbf{1}_n^{\top}\right)v^k,
\]
so 
\[
    \theta_i^{k+1} = \frac{1}{1 + \rho n} \left( v_i^k + \rho\sum_{j=1}^n v_j^k \right).
\]
The pairwise contributions cancel when summed over clients,
so $\sum_i v_i^k=\sum_i\vartheta_i$ can be precomputed. Computing each scalar
coordinate of $v_i^k$ takes $O(n)$ work, and these computations can be
parallelized across clients and coordinates. The complete iteration is
\[
    \begin{aligned}
        \theta_i^{k+1} &= \frac{1}{1 + \rho n} \left( v_i^k + \rho\sum_{j=1}^n v_j^k \right) \\
        z_{ij}^{k+1}   &= \mathbf{clip}_{ij} (\theta_i^{k+1} - \theta_j^{k+1} + u_{ij}^k) \\
        u_{ij}^{k+1}   &= u_{ij}^k + \theta_i^{k+1} - \theta_j^{k+1} - z_{ij}^{k+1}.
    \end{aligned}
\]
Each ADMM iteration requires $O(n^2p_L)$ total work and $O(n^2p_L)$ storage.
With parallel updates across pairs and parameter coordinates, and sequential
summation within each reduction, the span is $O(n+p_L)$.

\paragraph{Dual problem.} Associate an unrestricted multiplier
$y_{ij}\in\reals^{p_L}$ with each equality $\theta_i-\theta_j-z_{ij}=0$, for
$i<j$. Keeping $z_{ij}\in B_{ij}$ in the domain, the Lagrangian is
\[
    \mathcal{L}(\theta,z,y)= \frac{1}{2}\sum_{i=1}^{n}\|\theta_i-\vartheta_i\|^2 +\sum_{i<j}\braket{y_{ij},\theta_i-\theta_j-z_{ij}}.
\]
Define
\[
    a_i(y)=\sum_{j>i}y_{ij}-\sum_{j<i}y_{ji}.
\]
Minimization over $\theta$ gives $\theta_i=\vartheta_i-a_i(y)$, while
\[
    \inf_{z_{ij}\in B_{ij}}-\braket{y_{ij},z_{ij}} = -r_{ij}\|y_{ij}\|.
\]
Thus the dual problem is
\begin{equation}\label{eq:projection-dual}
        \underset{y_{ij}\in\mathbf{R}^{p_L}}{\mbox{maximize}}\quad
    q(y):=\sum_{i=1}^n\braket{\vartheta_i,a_i(y)}
    -\frac12\sum_{i=1}^n\|a_i(y)\|^2
    -\sum_{i<j}r_{ij}\|y_{ij}\|.
\end{equation}
For any feasible output-layer collection $\hat\theta$ and any $y$, weak duality
gives
\[
    0\leq P(\hat\theta)-P^\star
    \leq P(\hat\theta)-q(y)=:G(\hat\theta,y),
\]
where $P(\hat \theta)$ is the primal objective value at $\hat \theta$ and
$P^\star = \min_{\theta \in \mathcal{K}} P(\theta)$. The unscaled ADMM
multiplier is $y^k=\rho u^k$, so $G(\hat\theta,\rho u^k)$ is a computable upper
bound on the projection objective error. A finite ADMM iterate $\theta^k$ need
not be feasible: $z_{ij}^k\in B_{ij}$ alone does not imply
$\theta_i^k-\theta_j^k\in B_{ij}$. Feasibility must therefore be checked or
restored before using this certificate. Finally, the objective in
Assumption~\ref{asmp:proj} is scaled by $1/\eta$. Since the shared parameters
are projected exactly by averaging, their objective contribution cancels from
the suboptimality. Consequently, the sufficient stopping condition for that
assumption is
\[
    G(\hat\theta,\rho u^k)\leq\eta\delta,
\]
or, equivalently, $G(\hat\theta,\rho u^k)/\eta$ is a valid upper bound on the
scaled projection suboptimality.

\clearpage
\section{Differentially private Wasserstein embeddings}
\label{sec:private-embeddings}

We give a privacy guarantee for the one-time embedding release in
Section~\ref{ssec:data-embedding}. The guarantee concerns the numerical
embeddings sent to the server and their derived dissimilarities. It does not
require uniqueness, continuity, or stability of the optimal transport plan.

\paragraph{Public quantities and adjacency.}
Fix the nonzero dataset sizes $N_i$, the dimension $d$, and $N_0\geq 1$
reference points. The reference $\mathcal D_0$ and the feature--label
representation are public and independent of the private datasets. All records,
including their label coordinates, lie in a public Euclidean ball $B_2(0,R)$,
with $R>0$. This can be enforced by applying the fixed recordwise map
\[
    z\longmapsto \begin{cases}
        z,&\|z\|_2\leq R,\\
        Rz/\|z\|_2,&\|z\|_2>R
    \end{cases}
\]
before computing transport in a Euclidean representation. For categorical
labels, one can instead clip features to radius $R_X$ and retain a public label
encoding of norm at most $R_Y$, taking $R=\sqrt{R_X^2+R_Y^2}$. All distances
below refer to the resulting bounded records. 

Two collections $\boldsymbol{\mathcal D}=(\mathcal D_1,\ldots,\mathcal D_n)$
and $\boldsymbol{\mathcal D}'$ are \emph{client-replacement adjacent} if they
agree except for the contents of one client's dataset, whose size stays fixed.
This includes replacing a single record at a single client. It does not conceal
the public roster or dataset sizes. A release $\mathcal A$ is
$(\epsilon,\delta)$-differentially private for this adjacency if, for all
adjacent collections and measurable output events $S$,
\[
    \prob\{\mathcal A(\boldsymbol{\mathcal D})\in S\}
    \leq e^\epsilon
    \prob\{\mathcal A(\boldsymbol{\mathcal D}')\in S\}+\delta.
\]

\paragraph{Embedding normalization.}
Let $\pi_i(\mathcal D_i)$ be a coupling selected by a fixed measurable
deterministic rule, with nonnegative entries and marginals
\[
    \pi_i\ones_{N_i}=\frac{1}{N_0}\ones_{N_0},
    \qquad
    \pi_i^\top\ones_{N_0}=\frac{1}{N_i}\ones_{N_i}.
\]
In particular, it may be any optimal plan selected with a fixed rule for
resolving ties. Set
\[
    M_i=N_0\pi_i\mathcal D_i,\qquad
    E_i=\frac{M_i - \mathcal D_0}{N_0},\qquad
    D_{ij}=\|E_i-E_j\|_1.
\]

\begin{thm}[Private embedding and dissimilarity release]
\label{thm:private-embeddings}
    Under the conditions above, define
    \[
        \Delta=\frac{2R}{\sqrt N_0},\qquad
        c_{\epsilon,\delta}
        =\frac{\sqrt{\log(1/\delta)}
        +\sqrt{\log(1/\delta)+\epsilon}}{\sqrt{2}\,\epsilon},
    \]
    for any $\epsilon>0$ and $0<\delta<1$. Choose a fixed public noise
    standard deviation $\sigma\geq\Delta c_{\epsilon,\delta}$.
    Each client independently releases, once,
    \begin{equation}\label{eq:private-embedding-release}
        \widetilde E_i=E_i + H_i,\qquad
        (H_i)_{r\ell}\overset{\mathrm{iid}}{\sim}
        \mathcal N(0,\sigma^2),\qquad
        \sigma\geq\Delta c_{\epsilon,\delta}.
    \end{equation}
    All noise entries are independent across clients and of the data and
    public reference. Then the joint release of all $\widetilde E_i$ is
    $(\epsilon,\delta)$-differentially private under client replacement,
    and hence under replacement of one record at one client. The joint
    release of these embeddings together with the entire matrix
    \begin{equation}\label{eq:private-dissimilarities}
        \widetilde D_{ij}=\|\widetilde E_i-\widetilde E_j\|_1
    \end{equation}
    has the same guarantee. In particular, there is no additional privacy
    cost for computing all pairs. The sensitivity bound $\Delta$ is sharp
    for client replacement on the ball $B_2(0,R)$.
\end{thm}

\begin{proof}
    For each row $r$, the weights $N_0\pi_{i,rk}$ are nonnegative and sum to
    one. Thus $(M_i)_{r}$ is a convex combination of records in
    $B_2(0,R)$ and has Euclidean norm at most $R$. For any two datasets of
    client $i$, their barycentric matrices therefore satisfy
    \[
        \|E_i - E_i'\|_F^2
        =\frac{1}{N_0^2}\sum_{r=1}^{N_0}
          \|(M_i)_{r} - (M_i')_{r}\|_2^2
        \leq \frac{4R^2}{N_0} = \Delta^2.
    \]
    The public reference cancels, and both transport plans are allowed to
    change arbitrarily. For the vector $F$ formed by concatenating all client
    embeddings, only one client's block changes under adjacency. Consequently
    $\|F(\boldsymbol{\mathcal D})- F(\boldsymbol{\mathcal D}')\|_2\leq\Delta$,
    independently of $n$.

    We verify the stated Gaussian calibration directly. Write
    $v=F(\boldsymbol{\mathcal D}) - F(\boldsymbol{\mathcal D}')$,
    $u=\|v\|_2/\sigma$. If $p,p'$ are the densities of the two joint Gaussian
    releases, their privacy loss under $Y\sim p$ is
    \[
        \log\frac{p(Y)}{p'(Y)}
        \ \overset{\mathrm{law}}{=}\ \frac{u^2}{2}+uG,
        \qquad G\sim\mathcal N(0,1).
    \]
    The zero-shift case has zero privacy loss. Otherwise, the choice of
    $\sigma$ implies
    \[
        u\leq c_{\epsilon,\delta}^{-1}
        = \sqrt{2\log(1/\delta) + 2\epsilon} - \sqrt{2\log(1/\delta)},
        \qquad \frac{u^2}{2}+u\sqrt{2\log(1/\delta)}\leq\epsilon.
    \]
    Hence the probability that the privacy loss exceeds $\epsilon$ is at most
    $\prob\{G>\sqrt{2\log(1/\delta)}\}\leq e^{-\log(1/\delta)}=\delta$.
    Integrating $p\leq e^\epsilon p'$ outside this exceptional set gives the DP
    inequality for every event $S$. The argument applies to every ordered
    adjacent pair. If the public reference is random, the same inequality holds
    conditional on each reference and can be integrated over its
    data-independent distribution.

    Computing \eqref{eq:private-dissimilarities} is post-processing of the
    joint release, so it incurs no further privacy loss
    \citep[Proposition~2.1]{DworkRoth14}. Finally, take one client's records
    all equal to $Rv$ in one dataset and all equal to $-Rv$ in the other,
    for a unit vector $v$. Every feasible plan gives these same constant
    rows in $M_i$ and $M_i'$. Their embedding difference has Frobenius norm
    exactly $2R / \sqrt N_0$, proving sharpness.
\end{proof}

\end{document}